%% file: iclr2027_conference.tex
\documentclass{article} 
\usepackage{iclr2027_conference,times}
\usepackage{dsfont}
\usepackage{caption}
\input{math_commands.tex}

\usepackage{multirow}
\usepackage{booktabs}
\usepackage{graphicx}
\usepackage{subcaption}

\usepackage{hyperref}
\usepackage{url}
\usepackage{amsmath}
\usepackage{amssymb}
\usepackage{mathtools}
\usepackage{amsthm}
\usepackage{algorithm, algorithmic}
\usepackage{comment}

\theoremstyle{plain}
\newtheorem{theorem}{Theorem}
\newtheorem{proposition}{Proposition}
\newtheorem{lemma}{Lemma}
\newtheorem{corollary}{Corollary}
\theoremstyle{definition}
\newtheorem{definition}{Definition}

\theoremstyle{remark}
\newtheorem{remark}{Remark}

\iclrfinalcopy 
\begin{document}

\title{Matrix AdaGrad: Row-wise and Column-wise Adaptive Subgradient Methods}

\author{Wenpeng Zhang \\
Independent Reasearcher \\
\texttt{zhangwenpeng0@gmail.com} \\
\And
Runsheng Yu \\
Independent Reasearcher \\
\texttt{runshengyu@gmail.com} \\
\AND
Peilin Zhao \\
School of Artificial Intelligence \\
Shanghai Jiao Tong University\\
\texttt{peilinzhao@sjtu.edu.cn}
}

%

\maketitle
\lhead{Preprint}

\begin{abstract}
Adaptive optimization methods such as AdaGrad and Adam are widely used in modern deep neural network training, but their adaptive scaling is primarily designed for vector-valued parameters and does not explicitly exploit matrix structure. Recent matrix-aware optimizers demonstrate the benefits of structured optimization, yet a general theoretical framework for deriving matrix-aware adaptivity comparable to that of AdaGrad remains lacking. In this work, we develop an Online Mirror Descent framework with adaptive proximal functions for matrix-valued parameters, providing a principled methodology for deriving matrix-aware adaptive optimization through online regret minimization. By introducing row-wise and column-wise matrix proximal functions, our framework explicitly reveals the trade-off governing adaptive scaling: increasing the scaling factors reduces the gradient-dependent dual norm term while increasing the cost of evolving the proximal geometry. In the row-wise setting, this trade-off becomes separable under diagonal parameterization, allowing the adaptive scaling for each row to be derived independently by minimizing its corresponding row-wise regret bound. The column-wise counterpart follows directly by applying the row-wise construction to the transposed matrix. This framework yields Row-wise Matrix AdaGrad and Column-wise Matrix AdaGrad as concrete instantiations, with regret guarantees that are strictly tighter than those of entry-wise AdaGrad under row-sparse or column-sparse gradient structures. Experiments on matrix factorization and stacked deep MLP training further demonstrate the benefits of matrix-aware adaptive scaling, yielding improved optimization performance in both settings and enhanced optimization stability and trainability at larger learning rates and greater network depths in the latter.
\end{abstract}

\section{Introduction}

Adaptive optimization methods have emerged as a cornerstone of modern machine learning, playing a pivotal role in training large-scale deep neural networks. Unlike classical stochastic gradient descent (SGD), which applies a uniform learning rate to all parameters, adaptive optimizers dynamically adjust parameter-wise learning rates based on statistics of past gradients. By capturing heterogeneous optimization dynamics across parameters, these methods facilitate more efficient and stable training in high-dimensional parameter spaces. Among various adaptive optimization algorithms, Adam~\citep{kingma2015adam} and its decoupled weight decay variant, AdamW~\citep{loshchilov2019decoupled}, have become the de facto standard for training contemporary deep learning models, including large language models (LLMs)~\citep{touvron2023llama, team2025kimi, xu2026deepseek}, diffusion models~\citep{peebles2023scalable, arriola2025block}, recommendation models~\citep{pmlr-v235-zhai24a}, and beyond.

Underlying the empirical success of adaptive optimizers is a rich theoretical lineage rooted in AdaGrad~\citep{duchi2011adaptive, McMahanS10}, which established a principled interpretation of adaptivity through online convex optimization (OCO). Specifically, AdaGrad is derived within the online mirror descent (OMD) framework~\citep{orabona2019modern}, where proximal functions characterize the optimization geometry and regulate the gradient steps. The key insight is that gradients observed in earlier iterations provide valuable information about the underlying geometry of the optimization problem, and incorporating such knowledge into the proximal function enables more informative gradient-based learning. Rather than relying on a fixed proximal function, AdaGrad dynamically modifies the proximal function based on progressively accumulated gradient information. At each iteration, AdaGrad constructs the proximal function by minimizing the regret bound up to that point, yielding a sequence of data-dependent proximal functions whose regret guarantee is competitive with that of the best proximal function chosen in hindsight. In essence, AdaGrad frames adaptivity as the dynamic construction of optimization geometry from gradient history, providing a rigorous theoretical foundation for adaptive optimization.

Despite this theoretical rigor, classical adaptive optimizers, including AdaGrad, Adam, and AdamW, are designed primarily for vector-valued parameters, whereas most parameters in neural networks are inherently matrices. When applied to these matrix-valued parameters, such methods typically disregard their intrinsic matrix geometry, effectively treating them as flattened vectors of independent scalar entries. This limitation has motivated a growing body of research on matrix-aware optimization, which instead seeks to preserve and exploit the intrinsic geometry of matrix-valued parameters. Within this paradigm, methods such as Shampoo~\citep{gupta2018shampoo} and Muon~\citep{jordan2024muon} have demonstrated the practical benefits of leveraging matrix structure in large-scale model training. In particular, Muon has emerged as a compelling alternative to AdamW and has been widely adopted for training frontier open-source LLMs~\citep{team2025kimi,zeng2025glm,xu2026deepseek}. Furthermore, beyond exploiting matrix structure, recent research has increasingly explored various forms of adaptivity tailored to matrix-valued parameters. Examples include block-wise adaptation~\citep{zhang2025adam}, row-wise adaptation~\citep{li2025normuon}, and one-sided Shampoo preconditioners~\citep{pmlr-v267-xie25j,an2026asgo}. However, these designs are largely motivated by approximations to existing adaptive preconditioners, rather than being derived directly from first principles. Consequently, they lack a principled theoretical foundation comparable to that of AdaGrad, where the form of adaptivity emerges naturally from online regret minimization and the dynamic construction of proximal geometry.

In this work, we fill this theoretical gap by developing a general online mirror descent framework with adaptive proximal functions for matrix-valued parameters, within which matrix-aware adaptivity can be derived via online regret minimization in a principled manner. The first step in establishing this framework is to define appropriate proximal functions, which determine the underlying optimization geometry. A straightforward approach is to directly apply vector-valued proximal functions to the vectorized matrix parameters. However, this approach fails to preserve the intrinsic geometry of matrix-valued parameters and incurs prohibitive computational and memory costs. To avoid these limitations, we define proximal functions directly in the matrix space. Specifically, we introduce row-wise and column-wise matrix proximal functions defined as squared Mahalanobis-type norms over the rows or columns of a matrix, respectively. By treating each row or column as a coherent, structurally coupled unit rather than merely a collection of independent entries, these proximal functions capture both the scaling of individual rows or columns and the coupling between them. Computationally, the metric matrix inducing the Mahalanobis geometry is applied to the matrix variable within the Frobenius inner product via left multiplication for the row-wise geometry and right multiplication for the column-wise geometry. Next, we formalize the online mirror descent framework for matrix-valued decision variables, with an extended Bregman matrix divergence serving as the proximal regularizer in the OMD update rule. We then equip this OMD framework with adaptive matrix proximal functions that are allowed to evolve over iterations and establish a generic regret bound that provides the analytical foundation for deriving specific matrix-aware adaptive schemes.

To proceed, we first instantiate the matrix proximal function with a positive definite diagonal metric matrix, following the practical and computationally efficient choice commonly adopted in AdaGrad~\citep{duchi2011adaptive}. This choice also facilitates clean and intuitive analysis while avoiding excessive technical complexity. We next analyze the structure of the regret bound in detail, focusing on its three constituent terms and the role of each. The first term depends solely on initialization, while the remaining two terms depend explicitly on $T$ and thus govern the asymptotic behavior of the regret bound. The third term is given by the accumulated squared dual norms of the gradients, reflecting the cumulative cost of the observed gradients under the proximal geometry. The second term captures the differences between successive Bregman matrix divergences that arise from the temporal adaptation of the proximal function. In particular, if the proximal function is fixed, this term vanishes. Intuitively, allowing the proximal geometry to adapt dynamically to the observed gradients can reduce the dual-norm term, but such adaptation comes at a cost—specifically, it incurs an additional penalty quantified by the Bregman matrix divergence differences.

Through further analysis, focusing first on the row-wise setting, we show that the three terms in the regret bound can ultimately be reduced to two terms with opposing dependence on the adaptive scaling factors, revealing the fundamental trade-off that governs the evolution of the matrix proximal geometry. Specifically, one term grows linearly with the scaling factors, while the other decreases inversely with them. This opposing dependence indicates the existence of a retrospectively optimal proximal geometry that balances these two terms. Crucially, under the diagonal parameterization, the regret bound further decomposes into a fully row-wise separable form, allowing the optimal proximal scaling to be determined independently for each row. The associated per-row regret bound admits a simple two-term trade-off between its scaling factor and the corresponding reciprocal. Minimizing this bound via the arithmetic–geometric mean inequality, with equality attained when the two terms are balanced, yields the optimal scaling proportional to the square root of the accumulated squared $\ell_2$-norms of the corresponding row-wise gradients. This optimal scaling provides a concrete form for the adaptive scaling matrix, and consequently, a practical rule for evolving the matrix proximal geometry, leading to the Row-wise Matrix AdaGrad (Row-AdaGrad) algorithm. We then derive the final regret bound for Row-AdaGrad and provide an illustrative example showing that this bound is strictly tighter than that of standard AdaGrad when the gradient matrices are row-sparse. The column-wise counterpart follows analogously by applying the row-wise formulation to the transposed problem and then transposing back. We show that the column-wise proximal function is exactly the row-wise proximal function applied to the transposed variable, establishing a one-to-one correspondence that preserves the problem structure. The Column-wise Matrix AdaGrad (Column-AdaGrad) algorithm and its regret bound are thus obtained directly from the row-wise results, without requiring a separate derivation. Empirically, we demonstrate that matrix-aware adaptivity can improve both optimization performance and training stability. On matrix factorization, Row-AdaGrad substantially outperforms its entry-wise counterpart when the adaptive geometry aligns with the row-wise parameter structure. Moreover, in a challenging setting without normalization layers or residual connections, row-wise scaling enables effective training at greater depths and larger learning rates, while entry-wise scaling becomes unstable or fails entirely.

In summary, our work goes beyond the specific algorithms of Row-AdaGrad and Column-AdaGrad by providing a general framework for understanding and deriving matrix-aware adaptivity through the evolution of matrix proximal geometry.
Our analysis explicitly reveals a fundamental trade-off underlying adaptive optimization: dynamically adapting the geometry can reduce the cumulative dual-norm cost of the gradients, but the resulting temporal variation of the proximal geometry incurs an additional penalty. The specific form of adaptivity is governed by the regret trade-off between the benefit and cost of tailoring the proximal geometry. Under the diagonal parameterization, this trade-off admits an explicit and simple separable form that is readily amenable to analysis, making transparent how the resulting adaptive scaling is derived from the regret bound. In contrast, standard AdaGrad typically characterizes its adaptive scaling through a constrained hindsight optimization problem, motivated primarily by reducing the contribution of the dual-norm term~\citep{duchi2011adaptive}. This characterization provides less direct and intuitive insight into the fundamental mechanism underlying adaptivity. Our framework therefore provides a rigorous yet intuitive theoretical foundation for the design and analysis of matrix-aware adaptive optimizers, paving the way for more principled research in this area. In particular, new matrix proximal functions can serve as a principled starting point for deriving matrix-aware adaptive optimizers with different structural properties.

\section{Row-wise and Column-wise Matrix Proximal Functions}

In this section, we introduce matrix proximal functions for defining the optimization geometry of matrix-valued parameters. Specifically, we propose row-wise and column-wise matrix proximal functions. The row-wise proximal function induces a geometry over rows by weighting row inner products, whereas the column-wise counterpart operates analogously over columns. For comparison, we also review the standard proximal function for vector-valued parameters and its extension to vectorized matrix parameters, illustrating the limitations of direct vectorization in computational efficiency and in capturing intrinsic matrix structure.

\subsection{Vectorized Proximal Function}

For any vector $ x \in \mathbb{R}^d $, the vector proximal function associated with a positive definite metric (or scaling) matrix $A \in \mathbb{R}^{d \times d}$ ($ A \succ 0 $) is defined as
$$\psi_A(x) = \frac{1}{2} \|x\|_A^2 = \frac{1}{2} \langle x, Ax \rangle,$$
where $\langle \cdot, \cdot \rangle$ denotes the standard Euclidean inner product, and $\|x\|_A = \sqrt{\langle x, Ax \rangle}$ is the Mahalanobis norm induced by $A$. Expanding the quadratic form in terms of individual coordinates gives
$$\psi_A(x) = \frac{1}{2} \sum_{i=1}^d \sum_{j=1}^d A_{ij} x_i x_j,$$
where the diagonal entries of $A$ scale individual coordinates, while the off-diagonal entries encode pairwise interactions among them. This equips $\psi_A(x)$ with Mahalanobis geometry, generalizing the isotropic Euclidean metric to a problem-specific anisotropic metric that better captures the underlying data geometry and enables more effective optimization.

To define a proximal function for matrix parameters, a straightforward approach is to vectorize the matrix $X \in \mathbb{R}^{m \times n}$ into a vector in $\mathbb{R}^{mn}$ and then directly apply the vector proximal function. Note that matrix vectorization can be performed in multiple ways. Using the standard convention of column-stacking yields the following vectorized proximal function
\[
\psi_B^{\mathrm{vec}}(X) = \frac{1}{2} \|\mathrm{vec}(X)\|_B^2
= \frac{1}{2} \langle \mathrm{vec}(X), B\,\!\mathrm{vec}(X) \rangle,
\]
where $B \in \mathbb{R}^{mn \times mn}$ is a positive definite metric matrix that induces the Mahalanobis norm on the vectorized space $\mathbb{R}^{mn}$.

However, while vectorization provides a straightforward extension from vectors to matrices, it has several drawbacks. First, it treats the matrix as an unstructured one-dimensional vector, thereby obscuring its natural two-dimensional structure and the distinct roles of rows and columns, and making natural matrix operations—such as multiplication, trace, and spectral norms—difficult to express and exploit. Second, the vectorized representation can incur prohibitive computational and memory costs. Constructing and storing a full metric matrix $B \in \mathbb{R}^{mn \times mn}$ is impractical. Even when $B$ is restricted to a diagonal form, it still requires $mn$ parameters, leading to substantial computational and memory overhead for large matrices. 
To address these drawbacks, it is preferable to define proximal functions directly on the matrix space, preserving the inherent matrix structure while allowing for more efficient computation.

\subsection{Row-wise Proximal Function}

We introduce the \emph{row-wise proximal function}, which applies a Mahalanobis-type scaling directly to the rows of a matrix. This proximal function retains the matrix’s row-level structure by treating each row as a holistic unit, rather than as a mere collection of independent scalar entries. Unlike vectorized proximal functions, which flatten the matrix and obscure its inherent organization, the row-wise formulation instead induces a clear and interpretable row-wise geometry that explicitly captures both individual row scaling and inter-row interactions.

Formally, given a positive definite metric matrix $C \in \mathbb{R}^{m \times m}$, i.e., $ C \succ 0 $, for a matrix $X \in \mathbb{R}^{m \times n}$, the row-wise proximal function is defined as
\[
\psi_C(X) = \frac{1}{2} \|X\|_C^2 = \frac{1}{2} \langle X, CX \rangle_F,
\]
where \( \langle \cdot, \cdot \rangle_F \) denotes the Frobenius inner product, defined as $\langle X, Y \rangle_F = \sum_{i=1}^m \sum_{j=1}^n X_{ij} Y_{ij}$, and $\|X\|_C = \sqrt{\langle X, CX \rangle_F}$ is a Mahalanobis-type norm induced by $C$. This norm naturally extends the Mahalanobis norm from vectors to matrices, inducing a Mahalanobis-type geometry over the rows of \(X\). To illustrate the row-wise structure of $\psi_C(X)$ more explicitly, we first rewrite the Frobenius inner product in terms of rows as
\[
\langle X, Y \rangle_F = \sum_{i=1}^m \langle X^{(i)}, Y^{(i)} \rangle,
\]
where $X^{(i)}, Y^{(i)} \in \mathbb{R}^n$ denote the $i$-th rows of $X$ and $Y$, respectively, and \(\langle \cdot, \cdot \rangle\) is the standard Euclidean inner product on \(\mathbb{R}^n\). Leveraging this representation, the proximal function can be further expressed as a weighted sum of row inner products\footnote{Complete derivations of the row-wise and the following column-wise forms are provided in Appendix~\ref{derivation}.}
\[
\psi_C(X) = \frac{1}{2} \|X\|_C^2 = \frac{1}{2} \sum_{i=1}^m \sum_{j=1}^m C_{ij} \langle X^{(i)}, X^{(j)} \rangle.
\]
This form makes it explicit that the diagonal entries of $C$ scale the squared norms of the individual rows, while the off-diagonal entries weight the inner products between distinct rows. When $C = I_m$, the $m \times m$ identity matrix, $\psi_C(X)$ reduces to the standard squared Frobenius norm $\frac{1}{2}\|X\|_F^2$, and when $C = \operatorname{Diag}(c)$ with $c \in \mathbb{R}^m_{++}$, $\psi_C(X)$ scales each row independently by its corresponding diagonal entry $c_i$, without capturing any inter-row interactions.

\subsection{Column-wise Proximal Function}

Analogous to the row-wise proximal function, the \emph{column-wise proximal 
function} applies a Mahalanobis-type scaling directly to the columns of a 
matrix, treating each column as a holistic unit and thereby preserving the 
column-level structure. While the row-wise proximal function captures the 
structural relationships among rows, its column-wise counterpart captures the corresponding relationships among columns, thereby inducing a complementary geometry.

Given a positive definite metric matrix \(D \in \mathbb{R}^{n \times n}\), i.e., $ D \succ 0 $, for a matrix \(X \in \mathbb{R}^{m \times n}\), the column-wise proximal function is defined as
\[
\psi_D(X) = \frac{1}{2} \|X\|_D^2 = \frac{1}{2} \langle X, X D \rangle_F,
\]
where \(\|X\|_D = \sqrt{\langle X, X D \rangle_F}\) is the Mahalanobis-type norm induced by \(D\), which induces a Mahalanobis-type geometry over the columns of \(X\). We remark that the metric matrix \(C\) in the row-wise proximal function multiplies \(X\) from the left, whereas the metric matrix \(D\) in the column-wise proximal function multiplies \(X\) from the right. Analogous to the derivation in the row-wise case, expressing the proximal function explicitly in terms of the columns, we have
$$
\psi_D(X) = \frac{1}{2} \sum_{i=1}^n \sum_{j=1}^n D_{ij} \langle X_{(i)}, X_{(j)} \rangle,
$$
where \(X_{(i)}, X_{(j)} \in \mathbb{R}^m\) denote the \(i\)-th and \(j\)-th columns of \(X\), respectively, and \(\langle \cdot, \cdot \rangle\) denotes the standard Euclidean inner product on \(\mathbb{R}^m\). The diagonal entries of \(D\) scale the squared norms of individual columns, while the off-diagonal entries weight the inner products between distinct columns. When \(D = I_n\), \(\psi_D(X)\) reduces to the standard squared Frobenius norm \(\frac{1}{2}\|X\|_F^2\). More generally, when \(D = \mathrm{Diag}(d)\) with \(d \in \mathbb{R}^n_{++}\), the proximal function reduces to independently weighting the squared norm of each column by the corresponding diagonal entry \(d_j\), without introducing interactions between different columns.

\section{Online Mirror Descent for Matrices with Adaptive Matrix Proximal Functions}

In this section, we formalize the online learning framework upon which our theoretical analysis and algorithmic design are built.

\subsection{Online Mirror Descent with Matrix Decision Variables}

\textbf{Online Convex Optimization with Matrix Decision Variables}\,\, We work within the framework of Online Convex Optimization (OCO)~\citep{zinkevich2003online, hazan2019introduction, orabona2019modern}, specialized to matrix-valued decision variables. OCO can be viewed as a structured repeated game in which a learner and an adversary interact sequentially over $T$ rounds. At each round $t \in \{1, \dots, T\}$, the learner selects a decision matrix $X_t \in \mathcal{X}$ from a convex compact set $\mathcal{X} \subseteq \mathbb{R}^{m \times n}$. The adversary then reveals a convex loss function $f_t : \mathcal{X} \to \mathbb{R}$, and the learner suffers the loss $f_t(X_t)$. The learner's goal is to minimize the \emph{regret}, defined as the difference between the cumulative loss incurred by the learner and that of the best fixed matrix in hindsight:
\[
R(T) = \sum_{t=1}^T f_t(X_t) - \min_{X^* \in \mathcal{X}} \sum_{t=1}^T f_t(X^*).
\]
A meaningful regret is typically required to be \emph{sublinear}, i.e.,
$\lim_{T \to \infty} R(T)/T = 0$,
which implies that, for a sufficiently large number of rounds, the learner's performance approaches that of the optimal fixed decision chosen in hindsight with full knowledge of all rounds.

\textbf{Online Mirror Descent for Matrices}\,\, We now introduce Online 
Mirror Descent (OMD) with matrix decision variables, which serves as the 
foundation for designing our adaptive optimization algorithms. OMD is an 
online learning algorithmic template that generalizes standard gradient-based updates by incorporating a geometry-aware proximal term or regularizer defined through a strictly convex function, which encodes the intrinsic geometry of the decision space and guides the updates to respect the underlying problem structure. In the matrix setting, the proximal term is naturally expressed through the \emph{Bregman matrix divergence}, a generalization of the vanilla (vector) Bregman divergence.

\begin{definition}[Bregman Matrix Divergence]
Let $\mathcal{X} \subseteq \mathbb{R}^{m \times n}$ be a compact convex set with nonempty interior, and let $\psi: \mathcal{X} \to \mathbb{R}$ be a strictly convex and differentiable function on the interior $\mathop{\mathrm{int}} \mathcal{X}$. The \textbf{Bregman Matrix Divergence} with respect to $\psi$, denoted $B_\psi: \mathcal{X} \times \mathop{\mathrm{int}} \mathcal{X} \to \mathbb{R}$\footnote{Restricting $Y \in \mathop{\mathrm{int}} \mathcal{X}$ ensures that $\nabla \psi(Y)$ is well-defined, as $\psi$ may not be differentiable on the boundary.}, is defined as
\[
B_\psi(X; Y) = \psi(X) - \psi(Y) - \langle \nabla \psi(Y), X - Y \rangle_F,
\]
for $X \in \mathcal{X}$ and $Y \in \mathop{\mathrm{int}} \mathcal{X}$, where $\nabla \psi(Y) \in \mathbb{R}^{m \times n}$ is the gradient of $\psi$ at $Y$.
\end{definition}

$B_\psi$ quantifies the deviation between two matrices $X$ and $Y$ under the geometry induced by $\psi$. Similar to the vector case, $B_\psi$ is non-negative due to the strict convexity of $\psi$, and generally asymmetric, as it depends on the reference point $Y$. 

Then, the update rule of OMD for matrices at each round $t$ is given by
\[
X_{t+1} = \mathop{\mathrm{argmin}}_{X \in \mathcal{X}} \Big\{ \langle G_t, X \rangle_F + \frac{1}{\eta} B_\psi(X; X_t) \Big\},
\]
where $\eta > 0$ is the step size, and $G_t \in \partial f_t(X_t)$ denotes a subgradient of the convex loss function $f_t$ at $X_t$. Here, the term $B_\psi(X; X_t)$ acts as a \emph{proximity regularizer}, which encourages $X_{t+1}$ to stay close to the previous iterate $X_t$ under the geometry induced by $\psi$.

Instantiating different $\psi$ yields distinct variants of OMD for matrices, each reflecting a particular geometry associated with a corresponding proximity measure. While one could directly apply a vectorized proximal function to matrix variables, this approach disregards the intrinsic two-dimensional structure and fails to fully exploit the geometric properties of matrices. A more natural alternative is to leverage the row-wise and column-wise matrix proximal functions introduced in the previous section, enabling OMD to better exploit the structured geometries of matrix variables.

\subsection{Adaptive Matrix Proximal Functions}
\label{adaptive proximal}

In the standard OMD framework, the proximal function $\psi$ remains fixed 
throughout the entire learning process. While such a fixed geometry may 
capture certain structural properties of the problem, it cannot accommodate 
the evolving gradient over time. Adaptive subgradient methods, 
such as AdaGrad~\citep{duchi2011adaptive}, address this limitation by 
allowing the proximal function to vary across rounds---replacing $\psi$ with 
a sequence of data-dependent proximal functions $\psi_t$. This 
time-varying geometry enables the algorithm to adapt not only to the 
intrinsic structure of the problem, but also to the gradient information 
accumulated over time, which is key to the success of adaptive methods.

Inspired by the derivation of AdaGrad from OMD, designing matrix adaptive methods similarly requires an appropriate evolution mechanism for the matrix proximal function. At a high level, the starting intuition is analogous to that in the vector case: the regret bound of OMD for matrices depends on the choice of the proximal function $\psi_t$, and by appropriately designing the evolution of $\psi_t$ to better adapt to the revealed gradient information, we can obtain a tighter regret bound. To translate this insight into algorithmic design, we begin by presenting a template regret bound for OMD for matrices with a generic adaptive proximal function, which will guide the subsequent derivation of specific adaptive schemes.

\begin{corollary}
\label{OMD general bound}
Let $\{X_t\}_{t=1}^{T+1} \subseteq \mathop{\mathrm{int}} \mathcal{X}$ be the 
iterates of OMD for matrices with proximal functions $\psi_t$, and let 
$\eta > 0$ be the global step size. Let $\| \cdot \|$ be a norm on 
\(\mathbb R^{m\times n}\) and let $\| \cdot \|_*$ denote its dual norm. Suppose that, 
for each $t$, $\psi_t$ is $\lambda$-strongly convex with respect to 
$\| \cdot \|$ on $\mathcal{X}$. Let $G_t \in \partial f_t(X_t)$ be a 
subgradient of $f_t$ at $X_t$, and let 
$X^\star \in \arg\min_{X \in \mathcal{X}} \sum_{t=1}^T f_t(X)$. Then, we have
\[
\begin{split}
R(T) \leq\; & \frac{B_{\psi_1}(X^\star; X_1)}{\eta}
- \frac{B_{\psi_T}(X^\star; X_{T+1})}{\eta}
+ \frac{1}{\eta} \sum_{t=1}^{T-1}
\Big( B_{\psi_{t+1}}(X^\star; X_{t+1}) - B_{\psi_t}(X^\star; X_{t+1}) \Big) \\
&+ \frac{\eta}{2\lambda} \sum_{t=1}^T \|G_t\|_*^2.
\end{split}
\]
Moreover, since ${B_{\psi_T}(X^\star; X_{T+1})} \geq 0$, dropping the negative terminal term yields the following slightly looser regret bound
\[
R(T)
\;\leq\;
\frac{B_{\psi_1}(X^\star; X_1)}{\eta}
+
\frac{1}{\eta}
\sum_{t=1}^{T-1}
\Big(
B_{\psi_{t+1}}(X^\star; X_{t+1})
-
B_{\psi_t}(X^\star; X_{t+1})
\Big)
+
\frac{\eta}{2\lambda}
\sum_{t=1}^T \|G_t\|_*^2.
\]
\end{corollary}

Note that in the above corollary we use $\|\cdot\|$ and $\|\cdot\|_*$ to denote the primal and dual norms, respectively. In practice, they may vary with $t$ and be chosen according to the corresponding proximal function $\psi_t$. The complete proof of this corollary is provided in Appendix~\ref{general OMD proof}. In the following, we will carefully analyze this regret bound and derive the specific form of matrix adaptivity by minimizing it.

\section{Deriving Row-wise and Column-wise Matrix AdaGrad}

In this section, building on the results established in the preceding two sections, we present the derivation of the proposed Row-wise and Column-wise Matrix AdaGrad algorithms.

To begin with, we specify concrete instantiations of the matrix proximal functions. In the vector AdaGrad setting, the standard formulation employs a diagonal metric matrix to define the proximal function, which is much more popular and practical than its full-matrix counterpart, as the latter incurs substantially higher memory and computational costs~\citep{duchi2011adaptive, gupta2018shampoo}. Accordingly, we restrict our derivation to the diagonal case, which enables efficient row- or column-wise adaptive scaling while avoiding the high costs associated with full-matrix variants. In the following, we take the row-wise formulation as the basis for our derivation, with the column-wise case following analogously. Formally, for any $X \in \mathbb{R}^{m \times n}$, the row-wise matrix proximal function at iteration $t$ is given by $\psi_t^{H_t}(X) = \frac{1}{2} \|X\|_{H_t}^2 = \frac{1}{2} \langle X, H_t X \rangle_F,$ where $H_t = \mathrm{Diag}(s_t) \succ 0 \in \mathbb{R}^{m \times m}$ is a positive definite diagonal matrix, and $s_t = (s_{t,1}, \ldots, s_{t,m})^\top \in \mathbb{R}_{++}^m$ is the corresponding row-wise scaling vector. In analogy with the monotonicity of the scaling factors over iterations in vector AdaGrad, we further assume that the metric sequence is non-decreasing, i.e., \(H_{t+1}\succeq H_t\), or, equivalently \(s_{t+1,i}\ge s_{t,i}\) for all \(i\).

\subsection{Regret Analysis with Row-wise Matrix Proximal Functions}

Having established the explicit form of the proximal function, we now proceed to analyze the regret bound in Corollary~\ref{OMD general bound} in greater depth to derive further insights.

The regret bound consists of three main terms. The first term is a ``constant'' with respect to $T$, depending merely on the initialization---$X_1$ and $\psi_1$. In contrast, the remaining two terms grow with $T$ and thus dominate the asymptotic behavior of the bound. Among them, the last term has a relatively simple form, involving the accumulated squared dual norms of the subgradients, while the middle term is more intricate, reflecting the cumulative effect of the varying proximal functions through changes in the Bregman matrix divergence.

We proceed by examining the second term to uncover its underlying structure. The Bregman divergence difference $B_{\psi_{t+1}}(U; X_{t+1}) - B_{\psi_t}(U; X_{t+1})$ arises from the temporal adaptation of the proximal function, transitioning from $\psi_t$ to $\psi_{t+1}$—that is, the change of the diagonal matrix from $H_t$ to $H_{t+1}$. If the proximal function were fixed, i.e., $\psi_{t+1} = \psi_t$, this difference would vanish. Hence, it quantifies the additional penalty imposed by the adapted proximal function $\psi_{t+1}$ relative to $\psi_t$. To facilitate subsequent analysis, we introduce the trace notation: for a square matrix \(A\), \(\operatorname{Tr}(A) = \sum_{i} A_{i,i}\) denotes the sum of its diagonal entries. The Frobenius inner product can thus be written as \(\langle A, B \rangle_F = \operatorname{Tr}(A^\top B)\), leading to \(\psi_t^{H_t}(X) = \tfrac{1}{2} \operatorname{Tr}(X^\top H_t X)\), whose gradient is \(\nabla \psi_t^{H_t}(X) = H_t X\). The Bregman matrix divergence induced by \(\psi_t^{H_t}\) can then be expressed in trace form as 
\(B_{\psi_t^{H_t}}(X; Y)
= \tfrac{1}{2} \operatorname{Tr}(X^\top H_t X)
- \tfrac{1}{2} \operatorname{Tr}(Y^\top H_t Y)
- \operatorname{Tr}\bigl(Y^\top H_t (X - Y)\bigr).\)
This expression further simplifies to the compact form \(B_{\psi_t^{H_t}}(X; Y)
= \tfrac{1}{2} \|X - Y\|_{H_t}^2
= \tfrac{1}{2} \operatorname{Tr}\!\big((X - Y)^\top H_t (X - Y)\big).\)\footnote{For completeness, the derivation is provided in Appendix~\ref{appendix:trace-simplification}.} Since $H_t$ is diagonal, the trace expression simplifies to $B_{\psi_t^{H_t}}(X; Y) = \frac{1}{2} \sum_{i=1}^m s_{t,i} \, \|X_{i,:} - Y_{i,:}\|_2^2$, which is a weighted sum over the squared row norms. Using the compact trace representation and its row-wise form, the Bregman divergence difference can be written as
\[
\begin{aligned}
B_{\psi_{t+1}^{H_{t+1}}}(X^*; X_{t+1})
- B_{\psi_t^{H_t}}(X^*; X_{t+1})
&= \tfrac{1}{2} \operatorname{Tr}\!\big((X^* - X_{t+1})^\top (H_{t+1} - H_t) (X^* - X_{t+1})\big) \\
&= \tfrac{1}{2} \sum_{i=1}^m (s_{t+1,i} - s_{t,i})\, \|X^*_{i,:} - X_{t+1,i,:}\|_2^2.
\end{aligned}
\]
Summing over $t = 1, \dots, T-1$ and interchanging sums yields
$$\sum_{t=1}^{T-1} \Big( B_{\psi_{t+1}^{H_{t+1}}}(X^*; X_{t+1}) - B_{\psi_t^{H_t}}(X^*; X_{t+1}) \Big) = \frac{1}{2} \sum_{i=1}^m \sum_{t=1}^{T-1} (s_{t+1,i} - s_{t,i}) \, \|X^*_{i,:} - X_{t+1,i,:}\|_2^2.$$
The squared deviation of each row $i$ can be bounded by the largest squared deviation among all rows in the matrix
\[
\|X^*_{i,:} - X_{t+1,i,:}\|_2^2
\le \,\max_{\mathclap{\scriptscriptstyle{1 \le j \le m}}}\, \|X^*_{j,:} - X_{t+1,j,:}\|_2^2
= \|X^* - X_{t+1}\|_{:2,\infty}^2,
\]
where \(\|X\|_{:2,\infty} = \,\max_{\scriptscriptstyle{1 \le j \le m}}\!\|X_{j,:}\|_2\) is the row-wise mixed $(2,\infty)$ norm, with the $\ell_2$ norm applied along each row and the $\ell_\infty$ norm across rows, i.e., it measures the largest Euclidean norm among the rows of \(X\). Taking the maximum over all iterations $u = 1, \dots, T$, we obtain the uniform bound
\[
\|X^*_{i,:} - X_{t+1,i,:}\|_2^2
\le \|X^* - X_{t+1}\|_{:2,\infty}^2
\le \,\max_{\mathclap{\scriptscriptstyle{1 \le u \le T}}}\, \|X^* - X_u\|_{:2,\infty}^2.
\]
Since $s_{t+1,i} - s_{t,i} \ge 0$, we can first factor out the global row-wise supremum and then, by applying the telescoping sum, reduce the sum $\sum_{t=1}^{T-1} (s_{t+1,i} - s_{t,i})$ to $s_{T,i} - s_{1,i}$, yielding
\[
\begin{aligned}
\sum_{t=1}^{T-1} \Bigl( B_{\psi_{t+1}^{H_{t+1}}}(X^*; X_{t+1}) - B_{\psi_t^{H_t}}(X^*; X_{t+1}) \Bigr)
&\le \frac{1}{2}\,\max_{\mathclap{\scriptscriptstyle{1 \le u \le T}}}\, \|X^* - X_u\|_{:2,\infty}^2 \sum_{i=1}^m \sum_{t=1}^{T-1} (s_{t+1,i} - s_{t,i}) \notag\\
&\le \frac{1}{2} \,\max_{\mathclap{\scriptscriptstyle{1 \le u \le T}}}\, \|X^* - X_u\|_{:2,\infty}^2 \Big(\sum_{i=1}^m s_{T,i} - \sum_{i=1}^m s_{1,i}\Big).
\end{aligned}
\]
The first divergence term can be also bounded using the global row-wise supremum, giving
\[
B_{\psi_1^{H_1}}(X^*; X_1)
= \frac{1}{2} \sum_{i=1}^m s_{1, i} \,\!\|X^*_{i,:} - X_{1,i,:}\|_2^2
\le \frac{1}{2} \,\max_{\mathclap{\scriptscriptstyle{1 \le u \le T}}}\,\|X^* - X_u\|_{:2,\infty}^2\! \sum_{i=1}^m s_{1, i}.
\]
Combining the two preceding bounds, and using the trace notation $\sum_{i=1}^m s_{T,i} = \operatorname{Tr}(H_T)$, we obtain
\[
R(T) 
\le \frac{1}{2\eta} \,\max_{\mathclap{\scriptscriptstyle{1 \le u \le T}}}\, \|X^* - X_u\|_{:2,\infty}^2 \! \operatorname{Tr}(H_T) 
+ \frac{\eta}{2\lambda} \sum_{t=1}^T \|G_t\|_*^2.
\]
The current form of the regret bound does not yet rely heavily on the specific instantiation of the matrix proximal function $\psi_t^{H_t}$, apart from assuming that $H_t$ is diagonal and monotone. To make the bound more concrete, several additional steps are required to ultimately determine the explicit form of $\psi_t^{H_t}$ (or equivalently, $H_t$). We proceed by analyzing the squared dual norm term. So far, the regret bound has been expressed in terms of a generic norm $\|\cdot\|$ without specifying its explicit form. To accurately reflect the geometry induced by the proximal function, we specify it as the corresponding metric-induced matrix norm, whose squared form is $\|X\|_{H_t}^2 = \operatorname{Tr}(X^\top H_t X)$, with $H_t \succ 0$ as defined above. The proximal function $\psi_t^{H_t}(X) = \frac{1}{2}\|X\|_{H_t}^2$ is $1$-strongly convex with respect to $\|\cdot\|_{H_t}$, and hence satisfies the \(\lambda\)-strong convexity assumption in Corollary~\ref{OMD general bound}. The associated squared dual norm is then given by $\|G_t\|_{H_t}^{*2} = \|G_t\|_{H_t^{-1}}^2 = \operatorname{Tr}(G_t^\top H_t^{-1} G_t)$. This follows from the standard properties of matrix-induced norms and their duals; for completeness, a detailed derivation is provided in Appendix~\ref{app:induced_norms}. Accordingly, the accumulated squared dual norms of the subgradients in the regret bound are $\sum_{t=1}^T \|G_t\|_{H_t}^{*2} = \sum_{t=1}^T \operatorname{Tr}(G_t^\top H_t^{-1} G_t)$. Since $\{H_t\}_{t=1}^T$ evolves over time, directly analyzing the cumulative term
$\sum_{t=1}^T \operatorname{Tr}(G_t^\top H_t^{-1} G_t)$ can be technically intricate. To simplify the analysis, we consider a surrogate in which each time-varying metric $H_t$ is replaced by the terminal metric $H_T$; that is, we analyze
$\sum_{t=1}^T \operatorname{Tr}(G_t^\top H_T^{-1} G_t).$ Specifically, we assume the following bound holds
\[
\sum_{t=1}^T \operatorname{Tr}(G_t^\top H_t^{-1} G_t) \le a \sum_{t=1}^T \operatorname{Tr}(G_t^\top H_T^{-1} G_t),
\]
where \(a > 1\) is an abstract constant. The bound is possible because, with an appropriate choice of $\{H_t\}_{t=1}^T$, the cumulative effect of the time-varying metrics can be controlled by a scalar multiple of the same sum evaluated with the terminal metric. This bounding treatment largely simplifies the analysis, allowing us to maintain the main structure of the problem without delving into excessive technical complexities. Later, we will justify this bound by showing that, once an appropriate concrete adaptive proximal function is specified, the provisional bound $a$ can be replaced by a very small constant multiplicative factor independent of $T$. Now, by substituting the generic norm with the induced matrix norm and incorporating the surrogate trace bound with the abstract factor $a$, the regret bound can be expressed in the following form
\[
R(T) 
\le \frac{1}{2\eta} \,\max_{\mathclap{\scriptscriptstyle{1 \le u \le T}}}\, \|X^* - X_u\|_{:2,\infty}^2\! \operatorname{Tr}(H_T)   + \frac{\eta a}{2\lambda} \sum_{t=1}^T \operatorname{Tr}(G_t^\top H_T^{-1} G_t).
\]

\subsection{Retrospectively Optimal Row-wise Proximal}

Inspecting the above regret bound, we observe two components with opposing dependence on $H_T$: one term increases with $H_T$, while the other decreases through $H_T^{-1}$. This suggests a clear trade-off in the choice of $H_T$ and indicates the existence of a retrospectively optimal proximal function. In the following, we make this trade-off more explicit by decomposing the bound into a row-wise separable form. First, by definition of $H_T$, its inverse is given by
\(H_T^{-1} = \operatorname{Diag}(s_T) ^{-1}= \operatorname{Diag}(1/s_T),\) where \(1/s_T = (1/s_{T,1}, \ldots, 1/s_{T,m})^\top \in \mathbb{R}_{++}^m.\) For each $t$, the trace term can be written in a row-wise form as $\operatorname{Tr}(G_t^\top H_T^{-1} G_t) = \sum_{i=1}^m \|G_{t,i,:}\|_2^2 / s_{T,i}$. Summing over all $t$ and interchanging sums yields
\[
\sum_{t=1}^T \operatorname{Tr}(G_t^\top H_T^{-1} G_t)
= \sum_{t=1}^T \sum_{i=1}^m \frac{\|G_{t,i,:}\|_2^2}{s_{T,i}}
= \sum_{i=1}^m \frac{1}{s_{T,i}} \sum_{t=1}^T \|G_{t,i,:}\|_2^2.
\]
By combining this expression with $\operatorname{Tr}(H_T) = \sum_{i=1}^m s_{T,i}$, the regret bound can be written in a fully row-wise separable form
\[
R(T)
\le
\frac{\Delta_T^2}{2\eta} \sum_{i=1}^m s_{T,i} 
+ \frac{\eta a}{2\lambda} \sum_{i=1}^m \frac{V_i}{s_{T,i}} = \sum_{i=1}^m \left(
\frac{\Delta_T^2}{2\eta} s_{T,i} 
+ \frac{\eta a}{2\lambda} \frac{V_i}{s_{T,i}}
\right),\]
where $\Delta_T = \max_{\scriptscriptstyle{1 \le u \le T}}\! \|X^* - X_u\|_{:2,\infty}$, 
$V_i = \sum_{t=1}^T \|G_{t,i,:}\|_2^2$. Through this form, the trade-off between the two terms in the regret bound becomes more explicit: increasing $s_{T,i}$ enlarges the first term while reducing the second, and vice versa. This opposing dependence on $s_{T,i}$ allows us to derive the hindsight-optimal proximal function by minimizing the regret bound. Since the bound is fully separable across rows, we can optimize each row independently to achieve this. Specifically, for the $i$-th row, we consider its row-wise regret function w.r.t. $s_{T,i}$
\[
R_i(s_{T,i}) = \frac{\Delta_T^2}{2\eta} s_{T,i} + \frac{\eta a}{2\lambda} \frac{V_i}{s_{T,i}}, \quad s_{T,i} > 0.
\]
Since both terms are positive for $s_{T,i} > 0$, we can find the minimum by applying the arithmetic–geometric mean inequality
\[
\frac{\Delta_T^2}{2\eta} s_{T,i} + \frac{\eta a}{2\lambda} \frac{V_i}{s_{T,i}} 
\ge 2 \sqrt{\frac{\Delta_T^2}{2\eta} \cdot \frac{\eta a V_i}{2\lambda}} 
= \frac{\Delta_T\sqrt{a}}{\sqrt{\lambda}} \sqrt{V_i}.
\]
The minimum $\Delta_T\sqrt{a/\lambda} \sqrt{V_i}$ is attained when equality holds, i.e.,
\[
\frac{\Delta_T^2}{2\eta} s_{T,i} = \frac{\eta a}{2\lambda} \frac{V_i}{s_{T,i}} \quad \Longrightarrow \quad
s_{T,i}^* = \frac{\eta\sqrt{a}}{\Delta_T\sqrt{\lambda}}\sqrt{V_i}, \,\, i.e.,\,
s_{T,i}^* \propto \sqrt{\sum_{t=1}^T \|G_{t,i,:}\|_2^2}.
\]
Here, $\propto$ indicates that $s_{T,i}^*$ is proportional to the square root
of the accumulated squared gradient norms of the $i$-th row, highlighting the core row-wise scaling relationship by omitting constant factors that are identical across all rows. Extending the same reasoning to all time steps, for each $t = 1, \dots, T$, we set \(s_{t,i}^* \propto \sqrt{\sum_{k=1}^t \|G_{k,i,:}\|_2^2},\) which represents the hindsight-optimal row-wise scaling up to time $t$. To avoid numerical instabilities caused by extremely small values of $s_{t,i}^*$—which may arise in the initial stages of optimization when the $i$-th rows of past gradients $G_1,\dots,G_t$ are zero or nearly zero—we add a small positive stabilizer $\delta>0$. This also ensures that the resulting row-wise scaling vector is positive definite. We now have a concrete instantiation of $s_t$, i.e., $H_t$, which yields a practical evolution scheme for the adaptive row-wise matrix proximal, thereby giving rise to the Row-wise Matrix AdaGrad algorithm presented in Algorithm~\ref{alg:row-adagrad}. By construction, the scaling factors are non-decreasing in $t$, thus satisfying the monotonicity assumption imposed at the beginning of our derivation. Also note that upon introducing $\delta$ into each $s_t$, both $\operatorname{Tr}(H_T)$ and $\operatorname{Tr}(G_t^\top H_T^{-1} G_t)$ in the preceding regret bound now contain terms involving $\delta$. We will show that, with careful treatment, we can obtain a clean regret bound that reflects only the contribution of the subgradients, completely free of $\delta$.
\begin{algorithm}[t]
\caption{Row-wise Matrix AdaGrad \textbf{(Row-AdaGrad)}}
\label{alg:row-adagrad}
\begin{algorithmic}[1]
\STATE \textbf{Input:}\, Convex set \({\cal X}\), time horizon \(T\), step size $\eta > 0$, stabilizer $\delta > 0$
\STATE \textbf{Initialize:} $X_1 \in \mathcal{X} \subseteq \mathbb{R}^{m \times n}$, \,$s_0^{\mathrm{r}} = \delta\mathbf{1} \in \mathbb{R}^m$
\FOR {$t=1,\dots,T$}
\STATE Predict $X_t$ and receive loss function $f_t:\mathcal X\rightarrow\mathbb R$
\STATE Compute subgradient $G_t \in \partial f_t(X_t)$ of $f_t$ at $X_t$
\STATE Update row-wise scaling vector
\vspace{-0.5mm}
\[s_{t,i}^{\mathrm{r}} = \sqrt{\textstyle\sum_{k=1}^t \|G_{k,i,:}\|_2^2} + \delta,\,\,\,\! i=1,\dots,m\qquad\qquad\qquad\]
\vspace{-1.9mm}
\STATE Set row-wise matrix proximal $\psi_t^{H_t}(X) = \frac{1}{2} \langle X, H_t X \rangle_F$,\hspace{0.5mm} with $H_t = \mathrm{Diag}(s_t^{\mathrm{r}})$ 
\STATE Matrix Mirror Descent Update
\vspace{-1.1mm}
\[
X_{t+1} = \mathop{\mathrm{argmin}}_{X \in \mathcal{X}} \Big\{{\eta}\hspace{0.16mm}\langle G_t, X \rangle_F + B_{\psi_t^{H_t}}(X; X_t) \Big\}\qquad\qquad\quad
\]
\vspace{-2.3mm}
\ENDFOR
\end{algorithmic}
\end{algorithm}

\begin{remark} \label{trade} Our derivation is best understood in comparison with the classical derivation of entry-wise AdaGrad~\citep{duchi2011adaptive}. Their derivation is primarily centered on reducing the gradient-dependent dual-norm term. Writing \(s\) for the (entry-wise) scaling vector---with a slight abuse of notation---and \(g\) for the gradient vector, their scaling rule is obtained by solving the following hindsight problem
\begin{equation*}
\min_{s}\ \sum_{t=1}^{T}\sum_{i=1}^{d}\frac{g_{t,i}^{2}}{s_i}
\quad\text{s.t.}\quad s\succ 0,\ \langle\mathbf{1},s\rangle\le c .
\end{equation*}
Here, the budget constraint \(\langle\mathbf{1},s\rangle\le c\) is introduced externally to ensure well-posedness, while its constant \(c\) is subsequently absorbed into the step size. As a result, the particular form of the scaling rule is not directly revealed by the regret bound itself. In contrast, our derivation, though different in form, makes explicit where the adaptive form arises. Introducing a time-varying proximal function induces an intrinsic trade-off: it systematically suppresses the gradient-dependent dual-norm term, while simultaneously incurring a penalty associated with the temporal variation of the proximal functions. By expressing this trade-off in a row-wise separable form and minimizing the resulting regret upper bound with respect to the scaling factors, we obtain the proposed adaptive matrix scaling rule naturally in closed form. Thus, the scaling rule is not imposed by an external hindsight objective, but emerges as a direct consequence of the intrinsic trade-off governing the regret.
\end{remark}

In the following, we first address $\operatorname{Tr}(G_t^\top H_T^{-1} G_t)$. Using the derived concrete form of $\{H_t\}_{t=1}^{T}$, we can now quantify the relationship between $\sum_{t=1}^T \operatorname{Tr}(G_t^\top H_t^{-1} G_t)$ and $\sum_{t=1}^T \operatorname{Tr}(G_t^\top H_T^{-1} G_t)$ more precisely, deriving a bound with a concrete constant rather than an abstract one, thereby justifying the use of the abstract bound in the preceding analysis. Furthermore, we eliminate the dependence on $\delta$ in $H_T$ by deriving a bound on $\sum_{t=1}^T \operatorname{Tr}(G_t^\top H_T^{-1} G_t)$ that depends solely on the subgradients. This result is formalized in the following corollary.
\begin{corollary}
\label{key lemma like lemma4 in AdaGrad}
Let $\{G_t\}_{t=1}^{T}$ and $\{H_t\}_{t=1}^{T}$ (i.e., $\{s_t^{\mathrm{r}}\}_{t=1}^{T}$) be defined as in Algorithm~\ref{alg:row-adagrad}, with $s_{t,i}^{\mathrm{r}}= \delta + \sqrt{\sum_{k=1}^{t}\|G_{k,i,:}\|_2^2}$ and $H_t=\mathrm{Diag}(s_t^{\mathrm{r}})$.
Then the following bound holds
\[
\sum_{t=1}^T \operatorname{Tr}\bigl(G_t^\top H_t^{-1} G_t\bigr)
\le 2 \sum_{t=1}^T \operatorname{Tr}\bigl(G_t^\top H_T^{-1} G_t\bigr)
\le 2 \sum_{i=1}^{m} \sqrt{\sum_{t=1}^{T} \|G_{t,i,:}\|_2^2}.
\]
\end{corollary}
\begin{proof}
By definition $s_{t,i}^{\mathrm{r}} = \delta + \sqrt{\sum_{k=1}^{t}\|G_{k,i,:}\|_2^2}$, so that \(\|G_{t,i,:}\|_2^2 = (s_{t,i}^{\mathrm{r}} - \delta)^2 - (s_{t-1,i}^{\mathrm{r}} - \delta)^2\). Hence, the row-wise expansion of each trace can be expressed as
\[
\operatorname{Tr}(G_t^\top H_t^{-1} G_t)
= \sum_{i=1}^{m} \frac{\|G_{t,i,:}\|_2^2}{s_{t,i}^{\mathrm{r}}}
= \sum_{i=1}^{m} \frac{(s_{t,i}^{\mathrm{r}} - \delta)^2 - (s_{t-1,i}^{\mathrm{r}} - \delta)^2}{s_{t,i}^{\mathrm{r}}}.
\]
For each \(i\), we claim the following termwise inequality
\[
\frac{(s_{t,i}^{\mathrm{r}}-\delta)^2-(s_{t-1,i}^{\mathrm{r}}-\delta)^2}{s_{t,i}^{\mathrm{r}}} \le 2\!\left(\frac{(s_{t,i}^{\mathrm{r}}-\delta)^2}{s_{t,i}^{\mathrm{r}}}-\frac{(s_{t-1,i}^{\mathrm{r}}-\delta)^2}{s_{t-1,i}^{\mathrm{r}}}\right).
\]
By applying the difference-of-squares decomposition to the left-hand side and using $s_{t-1,i}^{\mathrm{r}} \le s_{t,i}^{\mathrm{r}}$ to replace $s_{t-1,i}^{\mathrm{r}}$ with $s_{t,i}^{\mathrm{r}}$ in the second factor, we obtain
\[
\text{LHS} 
= \frac{(s_{t,i}^{\mathrm{r}} - s_{t-1,i}^{\mathrm{r}})(s_{t,i}^{\mathrm{r}} + s_{t-1,i}^{\mathrm{r}} - 2\delta)}{s_{t,i}^{\mathrm{r}}}
\le \frac{2 (s_{t,i}^{\mathrm{r}} - s_{t-1,i}^{\mathrm{r}})(s_{t,i}^{\mathrm{r}} - \delta)}{s_{t,i}^{\mathrm{r}}}.
\]
Re-expressing the first factor as $s_{t,i}^{\mathrm{r}} - s_{t-1,i}^{\mathrm{r}} = (s_{t,i}^{\mathrm{r}} - \delta) - (s_{t-1,i}^{\mathrm{r}} - \delta)$, we then have
\[
\text{LHS}
\le \frac{2 (s_{t,i}^{\mathrm{r}} - s_{t-1,i}^{\mathrm{r}})(s_{t,i}^{\mathrm{r}} - \delta)}{s_{t,i}^{\mathrm{r}}}
= \frac{2 (s_{t,i}^{\mathrm{r}} - \delta)^2}{s_{t,i}^{\mathrm{r}}} - \frac{2 (s_{t-1,i}^{\mathrm{r}} - \delta)(s_{t,i}^{\mathrm{r}} - \delta)}{s_{t,i}^{\mathrm{r}}}.
\]
Noting that \( \frac{s_{t,i}^{\mathrm{r}} - \delta}{s_{t,i}^{\mathrm{r}}} = 1 - \frac{\delta}{s_{t,i}^{\mathrm{r}}} \ge \frac{s_{t-1,i}^{\mathrm{r}} - \delta}{s_{t-1,i}^{\mathrm{r}}} = 1 - \frac{\delta}{s_{t-1,i}^{\mathrm{r}}} \) due to \( s_{t,i}^{\mathrm{r}} \ge s_{t-1,i}^{\mathrm{r}} \), it follows that
\[
\text{LHS}
\le \frac{2 (s_{t,i}^{\mathrm{r}} - \delta)^2}{s_{t,i}^{\mathrm{r}}} - \frac{2 (s_{t-1,i}^{\mathrm{r}} - \delta)(s_{t,i}^{\mathrm{r}} - \delta)}{s_{t,i}^{\mathrm{r}}}
\le 2 \left( \frac{(s_{t,i}^{\mathrm{r}} - \delta)^2}{s_{t,i}^{\mathrm{r}}} - \frac{(s_{t-1,i}^{\mathrm{r}} - \delta)^2}{s_{t-1,i}^{\mathrm{r}}} \right),
\]
proving the claimed termwise inequality. Summing over $t=1,\dots,T$ and telescoping, we obtain
\[
\sum_{t=1}^{T} \text{LHS} 
\le 2 \sum_{t=1}^{T} \left( \frac{(s_{t,i}^{\mathrm{r}} - \delta)^2}{s_{t,i}^{\mathrm{r}}} - \frac{(s_{t-1,i}^{\mathrm{r}} - \delta)^2}{s_{t-1,i}^{\mathrm{r}}} \right)
= 2 \left( \frac{(s_{T,i}^{\mathrm{r}} - \delta)^2}{s_{T,i}^{\mathrm{r}}} - \frac{(s_{0,i}^{\mathrm{r}} - \delta)^2}{s_{0,i}^{\mathrm{r}}} \right).
\]
Since $s_{0,i}^{\mathrm{r}} = \delta$, the last term vanishes, leaving only the first term evaluated at $t = T$. Summing over $i = 1, \dots, m$ then yields the first inequality in the corollary.

Finally, by reordering sums and using \(s_{T,i}^{\mathrm{r}} = \delta + \sqrt{\sum_{t=1}^{T}\|G_{t,i,:}\|_2^2} \ge \sqrt{\sum_{t=1}^{T} \|G_{t,i,:}\|_2^2}\), we obtain
\[
2\sum_{t=1}^{T}\operatorname{Tr}(G_t^\top H_T^{-1}G_t) = 
2\sum_{i=1}^{m} \frac{\sum_{t=1}^{T}\|G_{t,i,:}\|_2^2}{s_{T,i}^{\mathrm{r}}} \le 2\sum_{i=1}^{m} \sqrt{\sum_{t=1}^{T}\|G_{t,i,:}\|_2^2}.
\]
\end{proof}

\begin{remark} \label{remark}
Our bound extends the corresponding entry-wise result of AdaGrad; see Lemma 4 of \citet{duchi2011adaptive}. When the matrix degenerates to a vector, it recovers their bound after a further simplification. In particular, our analysis retains the intermediate scaling-dependent term
$\sum_{t=1}^T \operatorname{Tr}\bigl(G_t^\top H_T^{-1}G_t\bigr),$
before reducing it to the standard cumulative gradient-norm bound. Thus, our bound provides a finer form than the final bound in Lemma 4. Moreover, unlike their induction-based proof, our bound follows directly from the proximal and regret structure, leading to a more transparent and intuitive proof.
\end{remark}

Next, we turn to handling $\operatorname{Tr}(H_T) = \sum_{i=1}^m s_{T,i}^{\mathrm{r}} = m\delta + \sum_{i=1}^m\!\sqrt{\sum_{t=1}^{T}\|G_{t,i,:}\|_2^2}$. The inclusion of $\delta$ suggests the appearance of an additional term \(\max_{\scriptscriptstyle{{1 \le u \le T}}}\! \|X^* - X_u\|_{:2,\infty}^2 m\delta\) in the regret bound. Recall that $\operatorname{Tr}(H_T)$ arises when bounding
\(\sum_{t=1}^{T-1} \Big( B_{\psi_{t+1}^{H_{t+1}}}(X^*; X_{t+1}) 
- B_{\psi_t^{H_t}}(X^*; X_{t+1}) \Big) 
+ B_{\psi_1^{H_1}}(X^*; X_1)\) within the regret analysis. Upon closer inspection, we observe that the preceding derivation employs a slightly relaxed version of the regret bound that omits the negative term $-B_{\psi_T^{H_T}}(X^*; X_{T+1})$ (see Corollary~\ref{OMD general bound}). In fact, this negative terminal term can naturally offset the additional contribution introduced by the stabilizer $\delta$, provided that $\delta$ is chosen sufficiently small to satisfy
$\max_{\scriptscriptstyle{1 \le u \le T}}\! \| X^* - X_u \|_{:2,\infty}^2 \, m \delta \le 2 B_{\psi_T^{H_T}}(X^*; X_{T+1})$. By reinstating this negative term into the analysis, the contribution of the stabilizer can be fully absorbed without inflating the overall regret bound. For completeness, a detailed derivation of a sufficient condition on $\delta$ is provided in Appendix~\ref{stabilizer}.

Combining the above results and the inequality in Corollary~\ref{key lemma like lemma4 in AdaGrad} with the preceding regret bound, we can now obtain the following theorem.
\begin{theorem}
\label{regret thm}
Let the sequences $\{X_t\}_{t=1}^{T}$ and $\{G_t\}_{t=1}^{T}$ be generated by Algorithm~\ref{alg:row-adagrad}. Then the following regret bound holds
\[
R(T) 
\le \frac{1}{2\eta} \,\max_{\mathclap{\scriptscriptstyle{1 \le u \le T}}}\, \|X^* - X_u\|_{:2,\infty}^2\! \sum_{i=1}^{m} \sqrt{\sum_{t=1}^{T}\|G_{t,i,:}\|_2^2}   + \eta \sum_{i=1}^{m} \sqrt{\sum_{t=1}^{T}\|G_{t,i,:}\|_2^2}.
\]
Then let \(\mathcal{D}_{:2,\infty} = \sup_{X,Y \in \mathcal{X}} \| X - Y \|_{:2,\infty}\) denote the diameter of \(\mathcal{X}\) under the row-wise mixed $(2,\infty)$ norm and set $\eta = \mathcal{D}_{:2,\infty}/\sqrt{2}$, we have
\[
R(T) 
\le \sqrt{2}\hspace{0.3mm}\mathcal{D}_{:2,\infty}\!\sum_{i=1}^{m} \sqrt{\sum_{t=1}^{T}\|G_{t,i,:}\|_2^2}.
\]
\end{theorem}

\begin{remark} \label{degeneration} Our analytical framework and the proposed matrix-aware algorithms constitute a natural extension of standard AdaGrad. Specifically, when the parameter matrix reduces to a single column, effectively degenerating into a vector, the accumulated squared $\ell_2$-norms of the row-wise gradients reduce to the accumulated squared coordinate-wise gradients. Consequently, the row-wise adaptive scaling in Row-AdaGrad reduces to the standard coordinate-wise scaling, thereby recovering the entry-wise AdaGrad update rule. Consistent with the algorithmic reduction, our derived regret bound reduces to that of entry-wise AdaGrad (see Corollary 1 and Corollary 6 in \citep{duchi2011adaptive}), with the row-wise mixed $(2,\infty)$ norm reducing to the $\ell_\infty$ norm. This establishes our formulation as a unified framework for matrix-valued optimization that encompasses standard entry-wise AdaGrad as a special case.
\end{remark}

\subsection{Row-AdaGrad Outperforms Entry-Wise AdaGrad: A Motivating Example}
\label{subsec:motivating}

We expect Row-AdaGrad to outperform entry-wise AdaGrad when the gradient
matrices are concentrated on a few rows (row-wise sparsity), or when the
entries within each row are strongly correlated (intra-row dependency).
We provide empirical evidence demonstrating the improved performance of Row-AdaGrad in Section~\ref{experiments}. Here we present an abstract example showing that when the gradient matrices are row-sparse,
the regret bound of Row-AdaGrad is strictly smaller than that of entry-wise
AdaGrad.\footnote{The benefits of exploiting intra-row dependency are more clearly illustrated empirically in Section~\ref{experiments}.}

Recall that when applied to matrices, entry-wise AdaGrad treats an $m\times n$
matrix as a vector in $\mathbb{R}^{mn}$ and applies independent per-coordinate
step sizes. Its regret bound takes the form
\[
R_{\mathrm{vec}}(T) \le \sqrt{2}\, \mathcal{D}_\infty
\sum_{i=1}^{m}\sum_{j=1}^{n} \sqrt{\sum_{t=1}^{T} G_{t,i,j}^2},
\]
where $\mathcal{D}_\infty = \sup_{X,Y\in\mathcal{X}}\|X-Y\|_\infty$ is the
diameter of $\mathcal{X}$ under the entry-wise $\ell_\infty$ norm.

We construct a scenario in which the data matrices are row-wise sparse, which leads to row-sparse gradients. We fix an integer $K>0$ and
set the horizon $T = Km$. For each row $i\in\{1,\dots,m\}$, define the
row-activated matrix $D^{(i)} = e_i \mathbf{1}_n^\top$, where $e_i$ is the
$i$-th standard basis vector in $\mathbb{R}^m$ and $\mathbf{1}_n$ is the
all-ones vector in $\mathbb{R}^n$, so that $D^{(i)}$ has ones in its $i$-th
row and zeros elsewhere. At each round $t = 1,\dots,T$, a single row is
activated cyclically: $i_t = ((t-1)\bmod m)+1$ and $D_t = D^{(i_t)}$. We consider the hinge loss
\[
f_t(X) = \max\big\{0,\; 1 - y_t \langle X, D_t\rangle_F\big\},
\]
where $y_t = (-1)^t$ is the label. We take the decision set to be the Frobenius norm ball
$\mathcal{X} = \{X \in \mathbb{R}^{m\times n} : \|X\|_F \le B\}$ with
radius $B < 1/\sqrt{n}$. Since $\|D_t\|_F = \sqrt{n}$, this choice ensures
that the margin is violated at every round, regardless of the player's
action: for all $t$ and all $X_t \in \mathcal{X}$,
\[
y_t \langle X_t, D_t\rangle_F
\;\le\; \big|\langle X_t, D_t\rangle_F\big|
\;\le\; \|X_t\|_F\|D_t\|_F
\;\le\; B\sqrt{n} \;<\; 1 .
\]
Consequently, the subgradient of $f_t$ at $X_t$ is
$G_t = -y_t D_t = (-1)^{t+1} D^{(i_t)}$, so that only the $i_t$-th row of
$G_t$ is nonzero, with all entries equal to $\pm 1$. By construction, each
row $i$ is activated exactly $K$ times, and a direct computation yields
\[
\sum_{i=1}^{m} \sum_{j=1}^{n} \sqrt{\sum_{t=1}^{T} G_{t,i,j}^2} = m n \sqrt{K},\quad\,\, \sum_{i=1}^{m} \sqrt{\sum_{t=1}^{T} \|G_{t,i,:}\|_2^2} 
= \sum_{i=1}^{m} \sqrt{\sum_{t=1}^{T} \sum_{j=1}^{n} G_{t,i,j}^2} 
= m \, \sqrt{n K}.
\]
Moreover, the two diameter measures are
$\mathcal{D}_\infty = \sup_{X,Y\in\mathcal{X}}\|X-Y\|_\infty = 2B$ and
$\mathcal{D}_{:2,\infty} = \sup_{X,Y\in\mathcal{X}}\|X-Y\|_{:2,\infty} = 2B$. Both suprema are attained, for example, by taking $X = BE_{11}$ and $Y = -BE_{11}$, where $E_{11}$ denotes the matrix whose only nonzero entry is $1$ at position $(1,1)$. Substituting the above quantities into the respective regret bounds gives
\[
R_{\mathrm{vec}}(T) \le 2\sqrt{2} B \cdot m n \sqrt{K}, \quad
R_{\mathrm{row}}(T) \le 2\sqrt{2} B \cdot m \, \sqrt{n K}.
\]
The ratio of these two regret upper bounds is \(\sqrt{n}\), showing that Row-AdaGrad performs better in this row-sparse setting by improving the regret bound by a factor of $\sqrt{n}$ compared to the entry-wise counterpart.

A closer examination of this construction reveals the mechanism underlying the regret-bound improvement. Entry-wise AdaGrad assigns an independent step size to each coordinate and consequently treats the entries within each row separately. In contrast, Row-AdaGrad aggregates the gradient information across the entries of each row, thereby adapting to the row-wise structure of the gradients. In the above construction, this reduces the per-row contribution to the cumulative gradient term from $n\sqrt{K}$ to $\sqrt{nK}$, resulting in a $\sqrt{n}$ improvement in the final regret bound. Remarkably, Row-AdaGrad achieves this strictly tighter bound while using significantly fewer adaptive scaling factors ($m$ instead of $mn$). This highlights the key insight: when optimizing over matrices, aligning the adaptivity with the inherent structural geometry of the gradients (here, row-wise sparsity) is more effective and more parameter-efficient than fine-grained but structurally blind entry-wise adaptation. Finally, note that for $n = 1$ the two algorithms coincide, consistent with the degeneracy discussed in Remark~\ref{degeneration}.

\subsection{Deriving Column-wise Matrix AdaGrad via Matrix Transposition}

To derive the \emph{column-wise} matrix AdaGrad and its regret bound, we do not need to re-derive all the results from scratch. Instead, we can leverage the intrinsic relationship between the row-wise and column-wise formulations. To this end, we first establish a one-to-one correspondence between the row-wise and column-wise matrix proximal functions via matrix transposition.

Recall that, for a positive definite diagonal matrix \(Q \in \mathbb{R}^{n \times n}\)
and any matrix \(X \in \mathbb{R}^{m \times n}\), 
the column-wise matrix proximal function is given by
\(\psi_Q^{\mathrm{col}}(X) = \frac{1}{2} \langle X, X Q \rangle_F.\) By the definition of the Frobenius inner product, for any two matrices \(A, B \in \mathbb{R}^{m \times n}\), we have \(
\langle A, B \rangle_F = \langle A^\top\!, B^\top \rangle_F.
\) Applying this property to \(\psi_Q^{\mathrm{col}}(X)\) and using the symmetry of \(Q\), we obtain
\[
\psi_Q^{\mathrm{col}}(X)
= \frac{1}{2} \langle X^\top\!,\,\! (X Q)^\top \rangle_F
= \frac{1}{2} \langle X^\top\!,\,\! Q X^\top \rangle_F.
\]
Noting that \(\tfrac{1}{2}\langle X^\top\!, Q X^\top \rangle_F = \psi_Q^{\mathrm{row}}(X^\top)\), it follows that \(\psi_Q^{\mathrm{col}}(X) = \psi_Q^{\mathrm{row}}(X^\top)\). This equality shows that the column-wise proximal function evaluated at \(X\) is exactly the row-wise proximal function evaluated at the transposed matrix \(X^\top\)\!, thus establishing a one-to-one correspondence between \(\psi_Q^{\mathrm{col}}(X)\) and \(\psi_Q^{\mathrm{row}}(X^\top)\) via matrix transposition. This correspondence reflects the intrinsic symmetry of the proximal geometry under transposition: transposing the decision matrix swaps its rows and columns, thereby converting row-wise geometry into column-wise geometry, or vice versa. Consequently, the column-wise matrix AdaGrad algorithm and its associated regret bound can be obtained directly from their row-wise counterparts through this transpositional equivalence.

To more precisely formalize this transpositional equivalence, we denote \(X' = X^\top \in \mathbb{R}^{n \times m}\) and, for any convex loss function \(f: \mathbb{R}^{m \times n} \to \mathbb{R}\), define its transposed counterpart \(f': \mathbb{R}^{n \times m} \to \mathbb{R}\) by \(f'(X') = f(X'^\top) = f(X)\). By construction, $f'$ is convex with respect to $X'$, and its subgradients are related to those of $f$ as follows: if $G \in \partial f(X)$, then $G' = G^\top \in \partial f'(X')$, and conversely, if $G' \in \partial f'(X')$, then $G = (G')^\top \in \partial f(X)$. Recall that we have established the equality \(\psi_Q^{\mathrm{col}}(X) = \psi_Q^{\mathrm{row}}(X').\) Moreover, since both proximal functions are differentiable, their gradients satisfy \(\nabla\psi_Q^{\mathrm{row}}(X') = (\nabla\psi_Q^{\mathrm{col}}(X))^\top.\)
These facts imply that the associated Bregman matrix divergences satisfy \(B_{\psi_Q^{\mathrm{col}}}(X; Y) = B_{\psi_Q^{\mathrm{row}}}(X'; Y')\). Hence, the column-wise mirror descent update \(X_{t+1} = \mathop{\mathrm{argmin}}_{X \in \mathcal{X}} \big\{ \eta \langle G, X \rangle_F + B_{\psi_Q^{\mathrm{col}}}(X; X_{t}) \big\}\) is equivalent, up to transposition, to the row-wise update on the transposed variable \(X'_{t+1} = \mathop{\mathrm{argmin}}_{X' \in \mathcal{X}'} \big\{ \eta \langle G', X' \rangle_F + B_{\psi_Q^{\mathrm{row}}}(X'; X'_{t}) \big\}\), where \(\mathcal{X}' = \{ X' \mid X' = X^\top\!,\, X \in \mathcal{X} \}\); that is, performing the row-wise update on \(X'\) and then transposing the result back yields exactly the same iterate as the column-wise update on \(X\).

The correspondences established above allow us to derive the column-wise matrix AdaGrad and its regret bound by applying the row-wise counterparts to the transposed variable and then transposing back. Central to this derivation is the design of the adaptive column-wise matrix proximal update rule. Specifically, when applying Row-AdaGrad to $X' = X^\top \in \mathbb{R}^{n \times m}$ (with $G'_t = G_t^\top$), the row-wise scaling vector and matrix are updated as
\[
s_{t,j}^{\mathrm{row}} \leftarrow \sqrt{\textstyle\sum_{k=1}^t \|G'_{k,j,:}\|_2^2} + \delta, \,\, j = 1, \dots, n, \quad\!
H_t = \mathrm{Diag}(s_t^{\mathrm{row}}) \in \mathbb{R}^{n \times n}.
\]
Noting that each row of $G'_t$ corresponds to a column of $G_t$, i.e., $G'_{t,j,:} = G_{t,:,j}$, letting $s_{t,j}^{\mathrm{col}} = s_{t,j}^{\mathrm{row}}$, $Q_t = H_t$, and transposing back immediately yields the column-wise scaling vector and matrix update rule
\[
s_{t,j}^{\mathrm{col}} \leftarrow \sqrt{\textstyle\sum_{k=1}^t \|G_{k,:,j}\|_2^2} + \delta, \quad j = 1, \dots, n, \quad
Q_t = \mathrm{Diag}(s_t^{\mathrm{col}}) \in \mathbb{R}^{n \times n},
\]
which computes the scaling by accumulating the squared column-wise gradient norms rather than the row-wise ones. The resulting Column-wise Matrix AdaGrad (Column-AdaGrad) algorithm is presented in Algorithm~\ref{alg:column-adagrad}.

\begin{algorithm}[t]
\caption{Column-wise Matrix AdaGrad (Column-AdaGrad)}
\label{alg:column-adagrad}
\begin{algorithmic}[1]
\STATE \textbf{Input:}\, Convex set \({\cal X}\), time horizon \(T\), step size $\eta > 0$, stabilizer $\delta > 0$
\STATE \textbf{Initialize:} $X_1 \in \mathcal{X} \subseteq \mathbb{R}^{m \times n}$, \,$s_0^{\mathrm{c}} = \delta\mathbf{1} \in \mathbb{R}^n$
\FOR {$t=1,\dots,T$}
\STATE Predict $X_t$ and receive loss function $f_t:\mathcal X\rightarrow\mathbb R$
\STATE Compute subgradient $G_t \in \partial f_t(X_t)$ of $f_t$ at $X_t$
\STATE Update column-wise scaling vector
\vspace{-0.5mm}
\[s_{t,j}^{\mathrm{c}} = \sqrt{\textstyle\sum_{k=1}^t \|G_{k,:,j}\|_2^2} + \delta,\,\,\,\! j=1,\dots,n\qquad\qquad\qquad\]
\vspace{-1.9mm}
\STATE Set column-wise matrix proximal $\psi_t^{Q_t}(X) = \frac{1}{2} \langle X,  X Q_t \rangle_F$,\hspace{0.5mm} with $Q_t = \mathrm{Diag}(s_t^{\mathrm{c}})$ 
\STATE Matrix Mirror Descent Update
\vspace{-1.1mm}
\[
X_{t+1} = \mathop{\mathrm{argmin}}_{X \in \mathcal{X}} \Big\{{\eta}\hspace{0.16mm}\langle G_t, X \rangle_F + B_{\psi_t^{Q_t}}(X; X_t) \Big\}\qquad\qquad\quad
\]
\vspace{-2.3mm}
\ENDFOR
\end{algorithmic}
\end{algorithm}

Next, we derive the regret bound for Column-AdaGrad. By definition, its corresponding regret is given as
\(R(T) = \sum_{t=1}^T f_t(X_t) - \min_{X^\star \in \mathcal{X}} \sum_{t=1}^T f_t(X^\star).\) By the transpositional equivalence, we have \(\sum_{t=1}^T f_t(X_t) = \sum_{t=1}^T f_t'(X_t'),\,\,
\min_{X^\star \in \mathcal{X}} \sum_{t=1}^T f_t(X^\star) = \min_{X^{\prime \star} \in \mathcal{X}'} \sum_{t=1}^T f_t'(X^{\prime \star}).\) Therefore, the regret of Column-AdaGrad can be equivalently written as
\[
R(T) = \sum_{t=1}^T f_t(X_t) - \min_{X^\star \in \mathcal{X}} \sum_{t=1}^T f_t(X^\star) 
= \sum_{t=1}^T f_t'(X_t') - \min_{X^{\prime \star} \in \mathcal{X}'} \sum_{t=1}^T f_t'(X^{\prime \star}).
\]
In other words, the regret of Column-AdaGrad on the original problem is exactly the regret of Row-AdaGrad on the transposed problem. Consequently, by applying the known Row-AdaGrad regret bound to the transposed problem, we have
\[
R(T) 
\le \frac{1}{2\eta} \,\max_{\mathclap{\scriptscriptstyle{1 \le u \le T}}}\, \|X^{\prime\star} - X_u'\|_{:2,\infty}^2 \sum_{j=1}^{n} \sqrt{\sum_{t=1}^{T}\|G_{t,j,:}'\|_2^2} 
+ \eta \sum_{j=1}^{n} \sqrt{\sum_{t=1}^{T}\|G_{t,j,:}'\|_2^2}.
\]
Noting that the row-wise mixed $(2,\infty)$ norm of $X'$ is exactly the column-wise mixed $(2,\infty)$ norm of $X$, and that the $j$-th row of $G_t'$ corresponds to the $j$-th column of $G_t$, transposing back to column notation immediately yields the regret bound for Column-AdaGrad.
\begin{theorem}
\label{regret-col}
Let the sequences $\{X_t\}_{t=1}^{T}$ and $\{G_t\}_{t=1}^{T}$ be generated by Algorithm~\ref{alg:column-adagrad}. Then the following regret bound holds
\[
R(T) 
\le \frac{1}{2\eta} \,\max_{\mathclap{\scriptscriptstyle{1 \le u \le T}}}\, \|X^* - X_u\|_{2:,\infty}^2\! \sum_{j=1}^{n} \sqrt{\sum_{t=1}^{T}\|G_{t,:,j}\|_2^2}   + \eta \sum_{j=1}^{n} \sqrt{\sum_{t=1}^{T}\|G_{t,:,j}\|_2^2}.
\]
Then let \(\mathcal{D}_{2:,\infty} = \sup_{X,Y \in \mathcal{X}} \| X - Y \|_{2:,\infty}\) denote the diameter of \(\mathcal{X}\) under the column-wise mixed $(2, \infty)$ norm and set $\eta = \mathcal{D}_{2:,\infty}/\sqrt{2}$, we have
\[
R(T) 
\le \sqrt{2}\hspace{0.3mm}\mathcal{D}_{2:,\infty}\!\sum_{j=1}^{n} \sqrt{\sum_{t=1}^{T}\|G_{t,:,j}\|_2^2}.
\]
\end{theorem}

Note that, similar to the Row-AdaGrad case, we can also construct an abstract example with sparse column gradients, in which Column-AdaGrad improves the regret bound by a factor of \(\sqrt{m}\) over entry-wise AdaGrad, demonstrating the benefits of exploiting matrix structure. The result can be obtained from the transpositional equivalence. However, a direct derivation in column notation is very short and intuitive, so we provide it in Appendix~\ref{col-example} without invoking this equivalence.

\section{Experiments}
\label{experiments}

To evaluate the benefits of exploiting matrix structure for adaptive optimization, we compare matrix-aware adaptive scaling methods with standard entry-wise adaptive methods across several tasks involving matrix-valued parameters. Our experiments consider the proposed Row-AdaGrad and Column-AdaGrad, as well as variants that combine row- or column-wise adaptive scaling with momentum, enabling a more direct comparison with Adam.

\subsection{Matrix Factorization}

We first consider matrix factorization as a representative task involving
matrix-valued parameters. Specifically, we use the MovieLens 100K data set~\citep{harper2015movielens}, which contains 100,000 ratings from 943 users for 1,682 movies. Each user and each movie is represented by a learnable embedding vector. By stacking these embeddings, the user and item parameters naturally form two matrices. Moreover, the embeddings within these matrices exhibit strong intra-dependence: the entries of each embedding vector jointly represent the latent characteristics of a user or an item. This structure provides a simple yet representative testbed for evaluating whether exploiting row-wise or column-wise structure can improve optimization.

\begin{figure}
    \centering
    \includegraphics[width=0.97\linewidth]{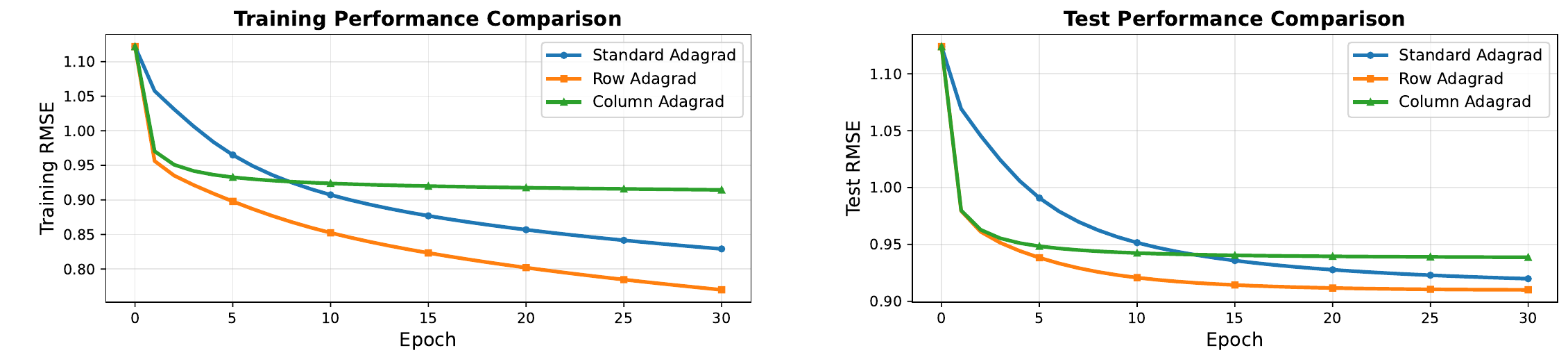}
    \caption{MovieLens 100K results with row-wise organization of user and item embeddings. Row-AdaGrad significantly outperforms entry-wise AdaGrad, while Column-AdaGrad performs the worst when its column-wise structure is misaligned with the parameter organization.}
    \label{fig1}
\end{figure}

We consider the following regularized matrix-factorization objective
\[
\min_{U,V}
\sum_{(i,j)\in\Omega}
\bigl(R_{ij}-u_i^\top v_j\bigr)^2
+\frac{\lambda_U}{2}\lVert U\rVert_F^2
+\frac{\lambda_V}{2}\lVert V\rVert_F^2,
\]
where $\Omega$ denotes the set of observed user-item pairs,
$U\in\mathbb{R}^{943\times k}$ and
$V\in\mathbb{R}^{1682\times k}$ are the user and item embedding matrices,
respectively, and $u_i$ and $v_j$ denote their corresponding embedding
vectors. We set the embedding dimension to $k=20$ and the regularization
coefficients to $\lambda_U=\lambda_V=0.02$. For each optimizer, we perform a grid search over a predefined set of learning rates and select the learning rate that achieves the best cross-validation performance. All methods use the same initialization and are trained for 30 epochs. The evaluation metric is Root Mean Squared Error (RMSE).

In this experiment, both the user and item embeddings are organized row-wise, so that each embedding vector corresponds to one row of the respective embedding matrix. As shown in Figure~\ref{fig1}, Row-AdaGrad
substantially outperforms standard entry-wise AdaGrad, demonstrating the benefit of adapting the scaling at the level of entire embedding vectors rather than individual parameters. In contrast, Column-AdaGrad performs the worst among the three methods, as its column-wise geometry is misaligned with the row-wise parameter organization. The performance gap between Row-AdaGrad and Column-AdaGrad highlights the importance of aligning the adaptive geometry with the underlying parameter organization. Importantly, this advantage stems from the alignment rather than the row-wise direction itself: if the same embeddings were organized column-wise, the relative performance of Row-AdaGrad and Column-AdaGrad would simply be reversed. An additional advantage of the proposed row- and column-wise scaling schemes is their substantially lower optimizer-state memory. Specifically, for an $m \times n$ parameter matrix, Row-AdaGrad and Column-AdaGrad maintain only a scaling vector in $\mathbb{R}^m$ or $\mathbb{R}^n$, respectively, requiring $O(m)$ or $O(n)$ optimizer state. In contrast, entry-wise AdaGrad maintains independent scaling factors for all $mn$ parameters, requiring $O(mn)$ optimizer state.

\subsection{Variable-Depth Stacked MLP Training}

To further investigate the effect of row-wise adaptive scaling on training stability\footnote{Column-wise scaling yields similar results; we report here only row-wise results for brevity.}, we consider a deliberately challenging setting that completely omits normalization and residual connections, i.e., without using common architectural techniques that can facilitate the optimization of deep neural networks~\citep{ba2016layer,he2016deep}. We simply stack MLP layers to increase the model depth, while keeping the learning rate and momentum coefficients fixed across optimizers. In particular, we use a relatively large learning rate, under which the choice of adaptive scaling becomes critical for stable optimization. We compare Row-wise Momentum, a momentum-based variant of Row-AdaGrad that replaces its cumulative gradient statistics with an exponential moving average (EMA), against Adam's entry-wise adaptive scaling under otherwise identical settings. For both optimizers, we set the learning rate to \(0.01\) and the momentum coefficients to \(\beta_1=\beta_2=0.9\). The dataset consists of \(N=640\) synthetic samples, with both the input \(X\in\mathbb{R}^{640\times20}\) and target \(Y\in\mathbb{R}^{640\times5}\) independently sampled from standard Gaussian distributions, i.e., \(X_{ij},Y_{ij}\sim\mathcal{N}(0,1)\). All hidden layers have width $20$, and the network depth is varied to study the effect of depth on training stability.

\begin{figure}
    \centering
    \includegraphics[width=0.97\linewidth]{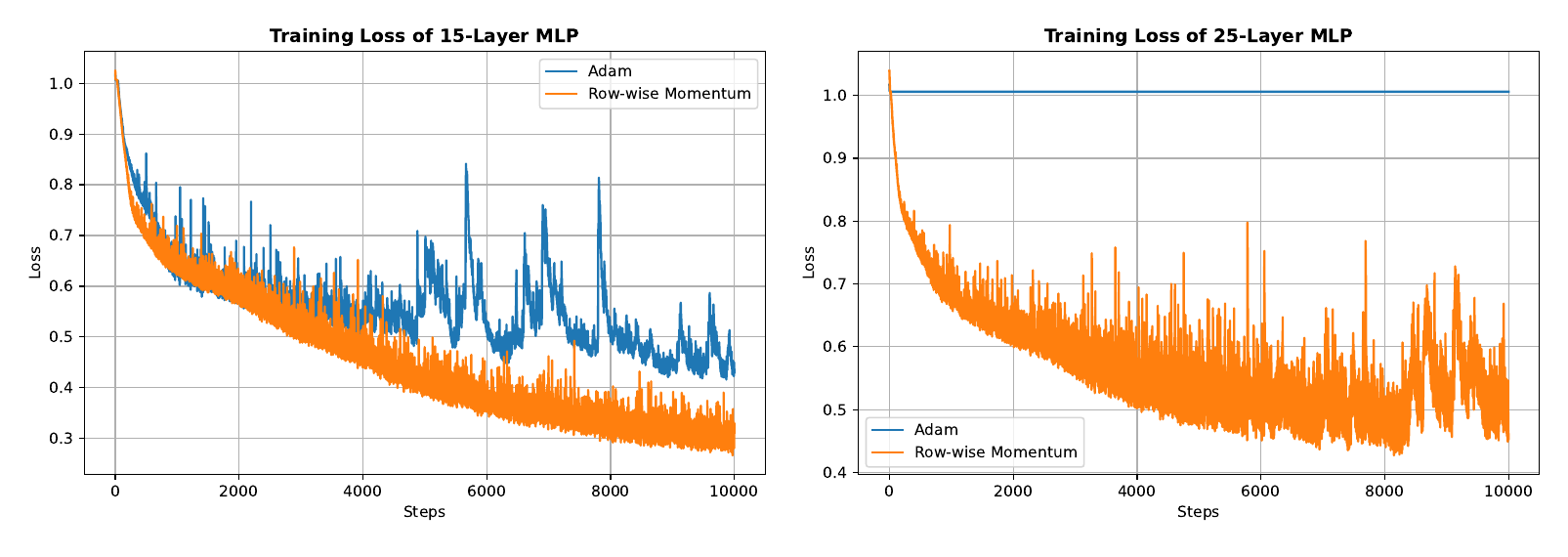}
    \caption{Training loss of Adam and Row-wise Momentum on stacked MLPs of different depths with \(lr=0.01\). Entry-wise scaling becomes highly unstable at $15$ layers and fails to train at $25$ layers, while row-wise scaling remains trainable.}
    \label{fig2}
\end{figure}

The results in Figure~\ref{fig2} show a substantial difference in
trainability between the two scaling schemes. Row-wise scaling remains
trainable as the network depth increases, supporting training at $25$ layers
despite increased loss fluctuations. In contrast, entry-wise scaling is
highly unstable at $15$ layers, exhibiting pronounced oscillations in the
training loss, and becomes completely untrainable at $25$ layers (with training already failing at $22$ layers). These results suggest that row-wise adaptive scaling leads to fundamentally different optimization dynamics and trainability from entry-wise scaling, rather than merely providing a different parameter-wise learning-rate schedule. In particular, row-wise scaling enables stable optimization at larger learning rates and supports the training of deeper networks, where entry-wise scaling becomes increasingly unstable and eventually fails to train. We further tested Adam with a smaller learning rate of \(0.001\), but it still completely fails to train at $25$ layers, indicating that the observed failure is not simply due to an overly large learning rate. More details of this experiment are provided in Appendix~\ref{small_lr}.

\section{Related Work}
\label{related}

The idea of adapting learning rates based on observed gradients dates back
to early work on adaptive and self-confident online learning~\citep{auer2002adaptive},
and was subsequently popularized by AdaGrad~\citep{duchi2011adaptive, McMahanS10}. A
central insight underlying AdaGrad is that adaptive scaling can be derived
from an online learning perspective through data-dependent regularization,
or equivalently, time-varying proximal functions in the framework of Online Mirror
Descent~\citep{orabona2019modern,jiang2023multi}. A comprehensive treatment of adaptive online learning and its connections to data-dependent regularization is provided by
McMahan~\citep{mcmahan2017survey}. Subsequent methods, including Adam and
AdamW~\citep{kingma2015adam,loshchilov2019decoupled}, have further established
adaptive gradient methods as de facto standards for optimizing modern deep neural networks. Our work builds on this adaptive proximal perspective and
extends classical entry-wise AdaGrad to matrix-valued parameters, with a focus on developing matrix-specific proximal geometries and their resulting adaptive scaling rules.

Beyond entry-wise adaptive scaling, various forms of structured adaptation have been explored. SM3~\citep{anil2019memory} shares second-moment statistics over a cover of tensor coordinates, while Blockwise Adaptive Gradient Descent~\citep{zheng2019blockwise} assigns a common adaptive scale to predefined parameter blocks. Adam-mini~\citep{zhang2025adam} similarly partitions parameters into blocks based on a block-diagonal approximation. Adafactor~\citep{shazeer2018adafactor} uses a factored representation of second-moment information to obtain a memory-efficient approximation to entry-wise adaptive scaling. These methods share two characteristics that distinguish them from our work. First, their motivation is largely efficiency-driven: they aim to reduce the memory overhead of maintaining full entry-wise second-moment statistics by sharing or factorizing adaptive quantities across parameters. Second, and more fundamentally, these methods operate on vector parameters rather than matrix-valued ones, while the matrix structure itself plays no explicit role in their design.

A distinct line of work explicitly incorporates the structure of matrix-valued parameters into adaptive optimization. A representative example is Shampoo~\citep{gupta2018shampoo}, which exploits matrix structure through Kronecker-factored preconditioning and can be viewed as an efficient approximation to the full-matrix AdaGrad geometry. More recently, one-sided variants of Shampoo have been explored~\citep{pmlr-v267-xie25j,an2026asgo}. NorMuon~\citep{li2025normuon}, building on the matrix orthogonalization framework of Muon, further introduces row-wise adaptive scaling to enhance orthogonalized updates. Our work is closely related to these approaches, but takes a different perspective by focusing on the underlying proximal geometry rather than directly specifying a structured preconditioner. Within a general online mirror descent framework, we define row-wise and column-wise proximal functions and derive the corresponding adaptive scaling rules from their proximal geometries through regret minimization. This formulation provides a more direct and interpretable account of the source of adaptivity, while also revealing the natural transpose symmetry between the row-wise and column-wise formulations. In contrast to one-sided Shampoo, which applies matrix preconditioning along only one mode, our framework accommodates both row-wise and column-wise choices. Moreover, whereas one-sided Shampoo retains a full matrix on the selected side and requires an explicit matrix-power computation, our diagonal structure yields row- or column-separable proximal geometries, eliminating the need for any matrix-power computation and substantially reducing both computational and memory costs. We derive matrix-specific regret guarantees and show through structured-gradient examples that matrix-level aggregation can yield strictly tighter bounds than entry-wise AdaGrad. Systematic experiments further demonstrate that matrix-aware adaptivity can improve training stability.

\section{Conclusion}
\label{final}

In this work, we develop a principled framework for deriving matrix-aware adaptive scaling from the proximal geometry of matrix-valued parameters. Within the Online Mirror Descent framework, we formulate row-wise and column-wise matrix proximal functions and show that the corresponding adaptive scaling rules naturally emerge through regret minimization. This perspective provides a more fundamental account of the source of adaptivity and yields matrix-specific regret guarantees, including strictly tighter bounds than entry-wise AdaGrad in structured-gradient settings. Our theoretical analysis is complemented by systematic experiments demonstrating that matrix-aware adaptivity can improve training stability. More broadly, our framework suggests a general approach to designing matrix-aware adaptive optimization methods: rather than specifying structured preconditioners directly, one can first design an appropriate matrix proximal function and then derive the corresponding adaptive scaling through regret minimization. This opens several directions for future work, particularly the systematic design of new matrix proximal functions that capture richer structures beyond the row-wise and column-wise geometries considered here. Such extensions could provide new matrix-aware adaptive methods with different structural inductive biases and theoretical guarantees. A more detailed comparison between row-wise and column-wise scaling, including their respective advantages, limitations, and suitability for different matrix structures, also remains an interesting direction for future study.

\bibliography{iclr2027_conference}
\bibliographystyle{iclr2027_conference}

\appendix
\section*{Appendix}

\section{Derivation of the Explicit Row-wise Form of 
\texorpdfstring{$\psi_C(X)$}{psi\_C(X)} and Column-wise Form of 
\texorpdfstring{$\psi_D(X)$}{psi\_D(X)}}
\label{derivation}

\subsection{Derivation of the explicit row-wise form of \texorpdfstring{$\psi_C(X)$}{psi\_C(X)}}
\label{app:proof_rowwise_proximal}

\begin{proof}
By definition of the row-wise proximal function, we have
\[
\psi_C(X) = \tfrac{1}{2} \|X\|_C^2 = \tfrac{1}{2} \langle X, CX \rangle_F.
\]
Using the row-wise form of the Frobenius inner product,
\[
\langle X, Y \rangle_F = \sum_{i=1}^m \langle X^{(i)}, Y^{(i)} \rangle,
\]
where \(X^{(i)}, Y^{(i)} \in \mathbb{R}^n\) denote the \(i\)-th rows of \(X\) and \(Y\), respectively, we can rewrite $\psi_C(X)$ as
\[
\psi_C(X) = \tfrac{1}{2} \sum_{i=1}^m \langle X^{(i)}, (CX)^{(i)} \rangle.
\]
Noting that the $i$-th row of $CX$ is given by $(CX)^{(i)} = \sum_{j=1}^m C_{ij} X^{(j)}$, it follows that
\[
\psi_C(X) = \tfrac{1}{2} \sum_{i=1}^m \Big\langle X^{(i)}, \sum_{j=1}^m C_{ij} X^{(j)} \Big\rangle.
\]
By linearity of the inner product, we obtain
\[
\psi_C(X) = \tfrac{1}{2} \sum_{i=1}^m \sum_{j=1}^m C_{ij} \langle X^{(i)}, X^{(j)} \rangle.
\]
This completes the derivation, showing the equivalence of the Frobenius representation and the explicit row-wise form of $\psi_C(X)$.
\end{proof}

\subsection{Derivation of the explicit column-wise form of \texorpdfstring{$\psi_D(X)$}{psi\_D(X)}}
\label{app:proof_columnwise_proximal}

\begin{proof}
By definition of the column-wise proximal function, we have
\[
\psi_D(X) = \tfrac{1}{2} \|X\|_D^2 = \tfrac{1}{2} \langle X, X D \rangle_F.
\]
Using the column-wise form of the Frobenius inner product,
\[
\langle X, Y \rangle_F = \sum_{j=1}^n \langle X_{(j)}, Y_{(j)} \rangle,
\]
where \(X_{(j)}, Y_{(j)} \in \mathbb{R}^m\) denote the \(j\)-th columns of \(X\) and \(Y\), respectively, we can rewrite $\psi_D(X)$ as
\[
\psi_D(X) = \tfrac{1}{2} \sum_{j=1}^n \langle X_{(j)}, (X D)_{(j)} \rangle.
\]
Noting that the \(j\)-th column of \(XD\) is given by \((X D)_{(j)} = \sum_{i=1}^n D_{ij} X_{(i)}\), it follows that
\[
\psi_D(X) = \tfrac{1}{2} \sum_{j=1}^n \Big\langle X_{(j)}, \sum_{i=1}^n D_{ij} X_{(i)} \Big\rangle.
\]
By linearity of the inner product, we obtain
\[
\psi_D(X) = \tfrac{1}{2} \sum_{i=1}^n \sum_{j=1}^n D_{ij} \langle X_{(i)}, X_{(j)} \rangle.
\]

This completes the derivation, showing the equivalence of the Frobenius representation and the explicit column-wise form of \(\psi_D(X)\).
\end{proof}

\section{Proof of Corollary~\ref{OMD general bound}}
\label{general OMD proof}

In this section, we provide the complete proof of Corollary~\ref{OMD general bound}. We first establish the one-step regret bound and then derive the final bound. The one-step bound relies on the three-point identity for Bregman matrix divergence stated below.

\begin{lemma}[Three-point Identity for Bregman Matrix Divergence] Let $B_\psi$ be the Bregman matrix divergence w.r.t. \(\psi: \mathcal{X} \to \mathbb{R}\), then for any three matrix points $X, Y \in \mathop{\mathrm{int}} \mathcal{X}$ and \(Z \in \mathcal{X}\), the following identity holds
$$B_\psi(Z; X) + B_\psi(X; Y) - B_\psi(Z; Y) = \langle \nabla \psi(Y) - \nabla \psi(X), Z - X \rangle_F.$$

\end{lemma}

\begin{proof}
By definition of $B_\psi$, we have
$$\begin{aligned}
B_\psi(Z; X) + B_\psi(X; Y) - B_\psi(Z; Y) = {} & \left[ \psi(Z) - \psi(X) - \langle \nabla \psi(X), Z - X \rangle_F \right] \\
& + \left[ \psi(X) - \psi(Y) - \langle \nabla \psi(Y), X - Y \rangle_F \right] \\
& - \left[ \psi(Z) - \psi(Y) - \langle \nabla \psi(Y), Z - Y \rangle_F \right].
\end{aligned}$$
The $\psi$ terms cancel as $\psi(Z) - \psi(X) + \psi(X) - \psi(Y) - \psi(Z) + \psi(Y) = 0$.
The inner product terms remain, so the left-hand side simplifies to
$$-\langle \nabla \psi(X), Z - X \rangle_F - \langle \nabla \psi(Y), X - Y \rangle_F + \langle \nabla \psi(Y), Z - Y \rangle_F.$$
Combining terms yield $$\langle \nabla \psi(X), X - Z \rangle_F + \langle \nabla \psi(Y), Z - X \rangle_F = \langle \nabla \psi(Y) - \nabla \psi(X), Z - X \rangle_F,$$ matching the right-hand side, proving this lemma.
\end{proof}

Next, we establish the following one-step regret bound.

\begin{lemma}[One-Step Regret Bound]
Let $\{X_t\}_{t=1}^{T+1} \subseteq \mathop{\mathrm{int}} \mathcal{X}$ be the iterates of OMD for matrices with proximal 
functions $\psi_t$, and let $\eta > 0$ be the global step size. Suppose 
that, for each $t$, $\psi_t$ is $\lambda$-strongly convex with respect to 
$\| \cdot \|$ on $\mathcal{X}$, and let $G_t \in \partial f_t(X_t)$ be a 
subgradient of $f_t$ at $X_t$. Then, for all $U \in \mathcal{X}$, we have
\[
\eta \bigl( f_t(X_t) - f_t(U) \bigr) \leq B_{\psi_t}(U; X_t) - B_{\psi_t}(U; X_{t+1}) + \frac{\eta^2}{2 \lambda} \|G_t\|_*^2,
\]
where $\| \cdot \|_*$ denotes the dual norm of $\| \cdot \|$.
\end{lemma}
\begin{proof}
Since $ G_t $ is a subgradient of the convex function $ f_t $ at $ X_t $, it follows that
$$f_t(U) \geq f_t(X_t) + \langle G_t, U - X_t \rangle_F.$$
Rearranging terms and multiplying both sides by $ \eta $, we obtain
$$\eta (f_t(X_t) - f_t(U)) \leq \eta \langle G_t, X_t - U \rangle_F.$$
Let $ J(X) = \langle G_t, X \rangle_F + \frac{1}{\eta} B_{\psi_t}(X; X_t) $. Then, the update $ X_{t+1} $ is the minimizer of $ J(X) $ over $ X \in \mathcal{X} $. The gradient of $ J(X) $ at $ X_{t+1} $ is
$$\nabla J(X_{t+1}) = G_t + \frac{1}{\eta} \nabla \psi_t(X_{t+1}) - \frac{1}{\eta} \nabla \psi_t(X_t).$$
By the first-order optimality condition for $ X_{t+1} $, we have
$$\langle G_t + \frac{1}{\eta}  \nabla \psi_t(X_{t+1}) - \frac{1}{\eta} \nabla \psi_t(X_t),  U - X_{t+1} \rangle_F \geq 0, \,\forall\, U \in \mathcal{X}.$$
We now express $ \eta \langle G_t, X_t - U \rangle_F $ using the three-point identity for Bregman matrix divergence. To this end, we first introduce $ X_{t+1} $ into the expression $ X_t - U $, which yields
$$\eta \langle G_t, X_t - U \rangle_F = \eta \langle G_t, X_t - X_{t+1} \rangle_F + \eta \langle G_t, X_{t+1} - U \rangle_F.$$
Then, rewriting the second term by adding and subtracting the gradient difference $\nabla \psi_t(X_{t+1}) - \nabla \psi_t(X_t)$ and rearranging, we get
\begin{align*}
\eta \langle G_t, X_t - U \rangle_F &= \langle \nabla \psi_t(X_{t+1}) - \nabla \psi_t(X_t), U - X_{t+1} \rangle_F\\ 
&\quad+ \langle \nabla \psi_t(X_t) - \nabla \psi_t(X_{t+1}) - \eta G_t, U - X_{t+1} \rangle_F 
 + \langle \eta G_t, X_t - X_{t+1} \rangle_F.
\end{align*}
From the optimality condition for $ X_{t+1} $ established earlier, the second term in the above equality is non-positive, leading to
$$\eta \langle G_t, X_t - U \rangle_F \leq \langle \nabla \psi_t(X_{t+1}) - \nabla \psi_t(X_t), U - X_{t+1} \rangle_F + \langle \eta G_t, X_t - X_{t+1} \rangle_F.$$
Applying the three-point identity for $ B_{\psi_t} $ with $X = X_{t+1}, Y = X_t$ and $Z = U$, we obtain
$$\langle \nabla \psi_t(X_{t+1}) - \nabla \psi_t(X_t), U - X_{t+1} \rangle_F = B_{\psi_t}(U; X_t) - B_{\psi_t}(U; X_{t+1}) - B_{\psi_t}(X_{t+1}; X_t).$$
Substituting this identity, we obtain
$$\eta \langle G_t, X_t - U \rangle_F \leq B_{\psi_t}(U; X_t) - B_{\psi_t}(U; X_{t+1}) - B_{\psi_t}(X_{t+1}; X_t) + \langle \eta G_t, X_t - X_{t+1} \rangle_F.$$
Applying Young's inequality for matrices (Lemma~\ref{Young inequality}), we obtain
$$\langle \eta G_t, X_t - X_{t+1} \rangle_F \leq \frac{\eta^2}{2 \lambda} \|G_t\|_* ^2 + \frac{\lambda}{2} \|X_t - X_{t+1}\|^2.$$
Then, since $ \psi_t $ is $ \lambda $-strongly convex w.r.t. \(\|\cdot\|\), it follows that
$$ B_{\psi_t}(X_{t+1}; X_t) \geq \frac{\lambda}{2} \|X_t - X_{t+1}\|^2.$$
Combining these inequalities, we get
$$\langle \eta G_t, X_t - X_{t+1} \rangle_F - B_{\psi_t}(X_{t+1}; X_t) \leq \frac{\eta^2}{2 \lambda} \|G_t\|_* ^2.$$
Thus, the bound for $ \eta \langle G_t, X_t - U \rangle_F $ simplifies to
$$\eta \langle G_t, X_t - U \rangle_F \leq B_{\psi_t}(U; X_t) - B_{\psi_t}(U; X_{t+1}) + \frac{\eta^2}{2 \lambda} \|G_t\|_* ^2.$$
This completes the proof of the one-step bound.
\end{proof}

\begin{lemma}[Young's Inequality for Matrices]
\label{Young inequality}
Let \(\mathbb{R}^{m \times n}\) be equipped with the Frobenius inner product. Let \(\|\cdot\|\) be a norm on \(\mathbb{R}^{m \times n}\) with its dual norm \(\|\cdot\|_*\) defined by
\[
\|Y\|_* = \sup_{\|X\| \leq 1} |\langle X, Y \rangle_F|.
\]
Then for any \(A, B \in \mathbb{R}^{m \times n}\) and scalar \(\lambda > 0\),
\[
\langle A, B \rangle_F \leq \frac{1}{2 \lambda} \|A\|_*^2 + \frac{\lambda}{2} \|B\|^2.
\]
\end{lemma}
\begin{proof}
By the definition of the dual norm, we have
\[
\langle A, B \rangle_F \leq \|A\|_* \cdot \|B\|.
\]
To bound the product \(\|A\|_* \|B\|\), consider the non-negative quadratic
\[
\left( \sqrt{\lambda} \|B\| - \frac{\|A\|_*}{\sqrt{\lambda}} \right)^2 \geq 0.
\]
Expanding the square yields
\[
\lambda \|B\|^2 - 2 \|A\|_* \|B\| + \frac{\|A\|_*^2}{\lambda} \geq 0.
\]
Rearranging terms, we get
\[
\|A\|_* \|B\| \leq \frac{1}{2\lambda} \|A\|_*^2 +  \frac{\lambda}{2} \|B\|^2.
\]
\end{proof}
Next, building on the one-step regret bound, we derive the final regret bound.
\begin{proof}
Dividing the one-step bound by $\eta$ and summing over $t=1,\dots,T$ gives
$$
\sum_{t=1}^T \big(f_t(X_t)-f_t(U)\big)
\leq
\frac{1}{\eta}
\sum_{t=1}^T
\Big(
B_{\psi_t}(U;X_t)-B_{\psi_t}(U;X_{t+1})
\Big)
+
\frac{\eta}{2\lambda}
\sum_{t=1}^T \|G_t\|_*^2.
$$
We first rearrange the sum of Bregman divergences as
$$
\begin{aligned}
\sum_{t=1}^T
\big(
B_{\psi_t}(U;X_t)-B_{\psi_t}(U;X_{t+1})
\big)
&=
B_{\psi_1}(U;X_1)-B_{\psi_T}(U;X_{T+1}) \\
&\quad+
\sum_{t=2}^T B_{\psi_t}(U;X_t)
-
\sum_{t=1}^{T-1} B_{\psi_t}(U;X_{t+1}).
\end{aligned}
$$
Re-indexing the first sum on the right-hand side by $t\mapsto t+1$, we obtain
$$
\sum_{t=2}^T B_{\psi_t}(U;X_t)
-
\sum_{t=1}^{T-1} B_{\psi_t}(U;X_{t+1})
=
\sum_{t=1}^{T-1}
\Big(
B_{\psi_{t+1}}(U;X_{t+1})
-
B_{\psi_t}(U;X_{t+1})
\Big).
$$

Therefore,
$$
\begin{aligned}
\sum_{t=1}^T \big(f_t(X_t)-f_t(U)\big)
&\leq
\frac{B_{\psi_1}(U;X_1)}{\eta}
-
\frac{B_{\psi_T}(U;X_{T+1})}{\eta} \\
&\quad+
\frac{1}{\eta}
\sum_{t=1}^{T-1}
\Big(
B_{\psi_{t+1}}(U;X_{t+1})
-
B_{\psi_t}(U;X_{t+1})
\Big)
+
\frac{\eta}{2\lambda}
\sum_{t=1}^T \|G_t\|_*^2.
\end{aligned}
$$
Since \(B_{\psi_T}(U;X_{T+1})\geq0\), the negative term $- B_{\psi_T}(U;X_{T+1})/{\eta}$ can be dropped to obtain the following regret bound
$$
\begin{aligned}
\sum_{t=1}^T \big(f_t(X_t)-f_t(U)\big)
&\leq
\frac{B_{\psi_1}(U;X_1)}{\eta}
+
\frac{1}{\eta}
\sum_{t=1}^{T-1}
\Big(
B_{\psi_{t+1}}(U;X_{t+1})
-
B_{\psi_t}(U;X_{t+1})
\Big) \\
&\quad+
\frac{\eta}{2\lambda}
\sum_{t=1}^T \|G_t\|_*^2.
\end{aligned}
$$
Since the above bound holds for every $U \in \mathcal{X}$, we take 
$U = X^\star \in \argmin_{X \in \mathcal{X}} \sum_{t=1}^T f_t(X)$, 
which yields the stated regret bound in Corollary~\ref{OMD general bound}. 
This completes the proof.
\end{proof}

\section{Proof of the Compact Form of the Bregman Matrix Divergence}
\label{appendix:trace-simplification}

\begin{proposition}
Let $\psi_t^{H_t}(X) = \frac{1}{2} \operatorname{Tr}(X^\top H_t X)$ with $H_t \in \mathbb{R}^{m \times m}$ symmetric and positive definite. Then, for any two matrices $X, Y \in \mathbb{R}^{m \times n}$, the Bregman matrix divergence induced by $\psi_t^{H_t}$ admits the compact form
\[
B_{\psi_t^{H_t}}(X; Y)
= \frac{1}{2} \|X - Y\|_{H_t}^2
= \frac{1}{2} \operatorname{Tr}((X - Y)^\top H_t (X - Y)).
\]
\end{proposition}

\begin{proof}
We start from the trace-based expression of the Bregman matrix divergence induced by $\psi_t^{H_t}$
\[
B_{\psi_t^{H_t}}(X; Y)
= \frac{1}{2} \operatorname{Tr}(X^\top H_t X)
- \frac{1}{2} \operatorname{Tr}(Y^\top H_t Y)
- \operatorname{Tr}\big(Y^\top H_t (X - Y)\big).
\]
Expanding the last term using linearity of the trace gives
\[
-\operatorname{Tr}\big(Y^\top H_t (X - Y)\big)
= -\operatorname{Tr}(Y^\top H_t X) + \operatorname{Tr}(Y^\top H_t Y).
\]
Substituting back and combining terms into a single trace, we obtain
\[
\begin{aligned}
B_{\psi_t^{H_t}}(X; Y) 
&= \frac{1}{2} \operatorname{Tr}(X^\top H_t X) - \frac{1}{2} \operatorname{Tr}(Y^\top H_t Y) 
   - \operatorname{Tr}(Y^\top H_t X) + \operatorname{Tr}(Y^\top H_t Y) \\
&= \frac{1}{2} \operatorname{Tr}(X^\top H_t X - 2 Y^\top H_t X + Y^\top H_t Y).
\end{aligned}
\]
Next, we separate the cross terms
\[
X^\top H_t X - 2 Y^\top H_t X + Y^\top H_t Y
= X^\top H_t X - Y^\top H_t X - Y^\top H_t X + Y^\top H_t Y.
\]
Since \(H_t\) is symmetric, i.e., $H_t^\top = H_t$, we have \((X^\top H_t Y)^\top = Y^\top H_t X\), which implies that \(\operatorname{Tr}(X^\top H_t Y) = \operatorname{Tr}(Y^\top H_t X)\). Using this, the expression can be rewritten as
\[
X^\top H_t X - X^\top H_t Y - Y^\top H_t X + Y^\top H_t Y.
\]
Recognizing this as the expansion of a quadratic form, we have
\[
(X - Y)^\top H_t (X - Y) = X^\top H_t X - X^\top H_t Y - Y^\top H_t X + Y^\top H_t Y.
\]
Hence,
\[
X^\top H_t X - 2 Y^\top H_t X + Y^\top H_t Y = (X - Y)^\top H_t (X - Y),
\]
and the Bregman matrix divergence can be expressed in compact quadratic form as
\[
B_{\psi_t^{H_t}}(X; Y) = \frac{1}{2} \operatorname{Tr}((X - Y)^\top H_t (X - Y)) = \frac{1}{2} \|X - Y\|_{H_t}^2.
\]
This completes the proof.
\end{proof}

\section{Derivation of the Dual Norm of the \texorpdfstring{$H_t$}{H\_t}-induced Matrix Norm}
\label{app:induced_norms}

\begin{proposition}
Let $H_t \in \mathbb{R}^{m \times m}$ be symmetric and positive definite. For any $X \in \mathbb{R}^{m \times n}$,  the $H_t$-induced matrix norm is $
\|X\|_{H_t} = \sqrt{\langle X, H_t X \rangle_F}$
, then for any $G \in \mathbb{R}^{m \times n}$, the corresponding dual norm is given by
\[
\|G\|_{H_t}^* = \sup_{X \neq 0} \frac{\langle G, X \rangle_F}{\|X\|_{H_t}} = \|G\|_{H_t^{-1}} = \sqrt{\operatorname{Tr}(G^\top H_t^{-1} G)}.
\]
Equivalently, the squared dual norm is
\[
\|G\|_{H_t}^{*2} = \|G\|_{H_t^{-1}}^2 = \operatorname{Tr}(G^\top H_t^{-1} G).
\]
\end{proposition}
\begin{proof}
By definition, the dual norm with respect to the Frobenius inner product is
\[
\|G\|_{H_t}^* = \sup_{X \neq 0} \frac{\langle G, X \rangle_F}{\|X\|_{H_t}}.
\]
To simplify the supremum, consider the matrix square root of $H_t$, denoted $H_t^{1/2}$, which satisfies $H_t = H_t^{1/2} H_t^{1/2}$. Defining $Y = H_t^{1/2} X$, we can invert this relation as $X = H_t^{-1/2} Y$, where $H_t^{-1/2} = (H_t^{1/2})^{-1}.$
This is well-defined since $H_t^{1/2}$ is invertible ($H_t \succ 0$). Then
\[
\langle G, X \rangle_F = \langle G, H_t^{-1/2} Y \rangle_F = \operatorname{Tr}\big(G^\top H_t^{-1/2} Y\big).
\]
Since $H_t^{-1/2}$ is symmetric (being the inverse of a symmetric matrix), i.e., $H_t^{-1/2} = (H_t^{-1/2})^\top$, we can rewrite the trace as
\[
\operatorname{Tr}\big(G^\top H_t^{-1/2} Y\big) = \operatorname{Tr}\big((H_t^{-1/2} G)^\top Y\big) = \langle H_t^{-1/2} G, Y \rangle_F.
\]
Moreover, substituting $X = H_t^{-1/2} Y$ into the $H_t$-induced norm gives
\[
\|X\|_{H_t} = \sqrt{\langle X, H_t X \rangle_F} = \sqrt{\langle H_t^{-1/2} Y, H_t (H_t^{-1/2} Y) \rangle_F}.
\]
Using $H_t H_t^{-1/2} = H_t^{1/2}$ (by the definition of the matrix square root), we have
\[
\langle H_t^{-1/2} Y, H_t H_t^{-1/2} Y \rangle_F = \langle H_t^{-1/2} Y, H_t^{1/2} Y \rangle_F = \operatorname{Tr}\big((H_t^{-1/2} Y)^\top (H_t^{1/2} Y)\big).
\]
Using the cyclic property of the trace, and the properties of $H_t^{-1/2}$, we get
\[
\operatorname{Tr}\big((H_t^{-1/2} Y)^\top (H_t^{1/2} Y)\big) = \operatorname{Tr}\big(Y^\top (H_t^{-1/2} H_t^{1/2}) Y\big) = \operatorname{Tr}(Y^\top Y) = \langle Y, Y \rangle_F = \|Y\|_F^2.
\]
Thus, the $H_t$-induced norm of $X$ reduces exactly to the standard Frobenius norm of $Y$
\[
\|X\|_{H_t} = \|Y\|_F.
\]
Hence, the dual norm becomes
\[
\|G\|_{H_t}^* = \sup_{Y \neq 0} \frac{\langle H_t^{-1/2} G, Y \rangle_F}{\|Y\|_F}.
\]
By the Cauchy--Schwarz inequality, the supremum is attained when $Y$ is a scalar multiple of $H_t^{-1/2} G$, which gives
\[
\|G\|_{H_t}^* = \sqrt{\langle H_t^{-1/2} G, H_t^{-1/2} G \rangle_F} = \sqrt{\operatorname{Tr}\big( (H_t^{-1/2} G)^\top (H_t^{-1/2} G) \big)}.
\]
By the symmetry of $H_t^{-1/2}$, we have
\[
\operatorname{Tr}\big( (H_t^{-1/2} G)^\top (H_t^{-1/2} G) \big) = \operatorname{Tr}\big( G^\top H_t^{-1/2} H_t^{-1/2} G \big).
\]
Since \(H_t^{1/2}\) is the square root of \(H_t\), we have
\[
H_t = H_t^{1/2} H_t^{1/2} \quad \implies \quad H_t^{-1} = H_t^{-1/2} H_t^{-1/2}.
\]
Therefore,
\[
\operatorname{Tr}\big( (H_t^{-1/2} G)^\top (H_t^{-1/2} G) \big) = \operatorname{Tr}\big( G^\top H_t^{-1} G \big) = \langle G, H_t^{-1} G \rangle_F.
\]
Hence, we obtain
\[
\|G\|_{H_t}^* = \|G\|_{H_t^{-1}} = \sqrt{\operatorname{Tr}\big( G^\top H_t^{-1} G \big)}.
\]
This completes the proof. Note that the derivation does not require $H_t$ to be diagonal.
\end{proof}

\section{A Sufficient Condition for Absorbing the Stabilizer-Induced Term}
\label{stabilizer}

In this section, we derive a sufficient condition under which the additional 
regret contribution induced by the stabilizer $\delta$ can be absorbed by the 
terminal Bregman term.

Specifically, it suffices to require
\[
\,\max_{\mathclap{\scriptscriptstyle{1 \le u \le T}}}\, \| X^* - X_u \|_{:2,\infty}^2 \, m \delta \;\le\; 2 B_{\psi_T^{H_T}}(X^*; X_{T+1}).
\]
Note that $B_{\psi_T^{H_T}}(X^*; X_{T+1})$ still contains $\delta$. Expressing it in row-wise form and substituting
\(s_{T,i} = \delta + \sqrt{\sum_{t=1}^{T} \| G_{t,i,:} \|_2^2},\)
we have
\begin{align*}
B_{\psi_T^{H_T}}(X^*\!; X_{T+1})
&= \frac{1}{2}\sum_{i=1}^m s_{T,i}
   \| X^*_{i,:} - X_{T+1,i,:} \|_2^2 \\
&= \frac{1}{2} \sum_{i=1}^m
   \bigg(\delta + \sqrt{\sum_{t=1}^{T} \| G_{t,i,:} \|_2^2}\bigg)
   \| X^*_{i,:} - X_{T+1,i,:} \|_2^2.
\end{align*}
Thus, it is sufficient to require
\[
\max_{\scriptscriptstyle{1\le u \le T}}\!\| X^* - X_u \|_{:2,\infty}^2 \, m \delta \;\le\; \sum_{i=1}^m \sqrt{\sum_{t=1}^{T} \| G_{t,i,:} \|_2^2} \, \| X^*_{i,:} - X_{T+1,i,:} \|_2^2.
\]
Solving for $\delta$ yields an explicit upper bound
\[
\delta \;\le\; { \sum_{i=1}^m \sqrt{\sum_{t=1}^{T} \| G_{t,i,:} \|_2^2} \, \| X^*_{i,:} - X_{T+1,i,:} \|_2^2} \Big/ {\Big(m\max_{1\le u \le T} \| X^* - X_u \|_{:2,\infty}^2}\Big).
\]
Hence, whenever this condition holds, the stabilizer-induced regret 
contribution can be fully absorbed by the negative terminal Bregman term, 
so that no additional explicit $\delta$-dependent penalty remains in the 
resulting regret bound.

Even when the above sufficient condition is not imposed, the presence of the
stabilizer does not affect the asymptotic order of the regret bound. In
particular, if $\delta = O(1)$ and the iterates remain in a bounded domain,
the additional stabilizer-induced contribution
$m\delta \max_{1 \le u \le T} \|X^* - X_u\|_{:2,\infty}^2$
is $O(1)$ and hence does not alter the leading asymptotic order of a
sublinear regret bound. Thus, the condition above provides a stronger
guarantee of complete absorption, while substantially weaker conditions are
sufficient to preserve the leading asymptotic order of the regret bound.

\section{Column-AdaGrad Outperforms Entry-Wise AdaGrad: A Motivating Example}
\label{col-example}

Recall that when applied to matrices, entry-wise AdaGrad treats an $m\times n$
matrix as a vector in $\mathbb{R}^{mn}$ and applies independent per-coordinate
step sizes. Its regret bound takes the form
\[
R_{\mathrm{vec}}(T) \le \sqrt{2}\, \mathcal{D}_\infty
\sum_{i=1}^{m}\sum_{j=1}^{n} \sqrt{\sum_{t=1}^{T} G_{t,i,j}^2},
\]
where $\mathcal{D}_\infty = \sup_{X,Y\in\mathcal{X}}\|X-Y\|_\infty$ is the
diameter of $\mathcal{X}$ under the entry-wise $\ell_\infty$ norm.

We construct a scenario in which the data matrices are column-wise sparse, which leads to column-sparse gradients. We fix an integer $K>0$ and set the horizon $T = Kn$. For each column $j\in\{1,\dots,n\}$, define the column-activated matrix $D_{(j)} = \mathbf{1}_m e_j^\top$, where $e_j$ is the $j$-th standard basis vector in $\mathbb{R}^n$ and $\mathbf{1}_m$ is the all-ones vector in $\mathbb{R}^m$, so that $D_{(j)}$ has ones in its $j$-th column and zeros elsewhere. At each round $t = 1,\dots,T$, a single column is activated cyclically: $j_t = ((t-1)\bmod n)+1$ and $D_t = D_{(j_t)}$. We consider the hinge loss

$$
f_t(X) = \max\big\{0,\; 1 - y_t \langle X, D_t\rangle_F\big\},
$$
where $y_t = (-1)^t$ is the label. We take the decision set to be the Frobenius norm ball
$\mathcal{X} = \{X \in \mathbb{R}^{m\times n} : \|X\|_F \le B\}$ with
radius $B < 1/\sqrt{m}$. Since $\|D_t\|_F = \sqrt{m}$, this choice ensures
that the margin is violated at every round, regardless of the player's
action: for all $t$ and all $X_t \in \mathcal{X}$,
$$
y_t \langle X_t, D_t\rangle_F
\;\le\; \big|\langle X_t, D_t\rangle_F\big|
\;\le\; \|X_t\|_F\|D_t\|_F
\;\le\; B\sqrt{m} \;<\; 1 .
$$
Consequently, the subgradient of $f_t$ at $X_t$ is
$G_t = -y_t D_t = (-1)^{t+1} D_{(j_t)}$, so that only the $j_t$-th column of
$G_t$ is nonzero, with all entries equal to $\pm 1$. By construction, each
column $j$ is activated exactly $K$ times, and a direct computation yields
$$
\sum_{i=1}^{m} \sum_{j=1}^{n} \sqrt{\sum_{t=1}^{T} G_{t,i,j}^2}
= m n \sqrt{K},\quad
\sum_{j=1}^{n} \sqrt{\sum_{t=1}^{T} \|G_{t,:,j}\|_2^2}
= \sum_{j=1}^{n} \sqrt{\sum_{t=1}^{T} \sum_{i=1}^{m} G_{t,i,j}^2}
= n \, \sqrt{m K}.
$$
Moreover, the two diameter measures are
$\mathcal{D}_\infty = \sup_{X,Y\in\mathcal{X}}\|X-Y\|_\infty = 2B$ and
$\mathcal{D}_{2:,\infty} = \sup_{X,Y\in\mathcal{X}}\|X-Y\|_{2:,\infty} = 2B$. Both suprema are attained, for example, by taking $X = BE_{11}$ and $Y = -BE_{11}$, where $E_{11}$ denotes the matrix whose only nonzero entry is $1$ at position $(1,1)$. Substituting the above quantities into the two respective regret bounds gives
$$
R_{\mathrm{vec}}(T) \le 2\sqrt{2} B \cdot m n \sqrt{K}, \quad
R_{\mathrm{col}}(T) \le 2\sqrt{2} B \cdot n \, \sqrt{m K}.
$$
The ratio of these two regret upper bounds is \(\sqrt{m}\), showing that Column-AdaGrad achieves a $\sqrt{m}$-factor improvement in the regret bound over the entry-wise counterpart in this column-sparse setting.

\section{Variable-Depth Stacked MLP Training at a Smaller Learning Rate}
\label{small_lr}

\begin{figure}
    \centering
    \includegraphics[width=0.97\linewidth]{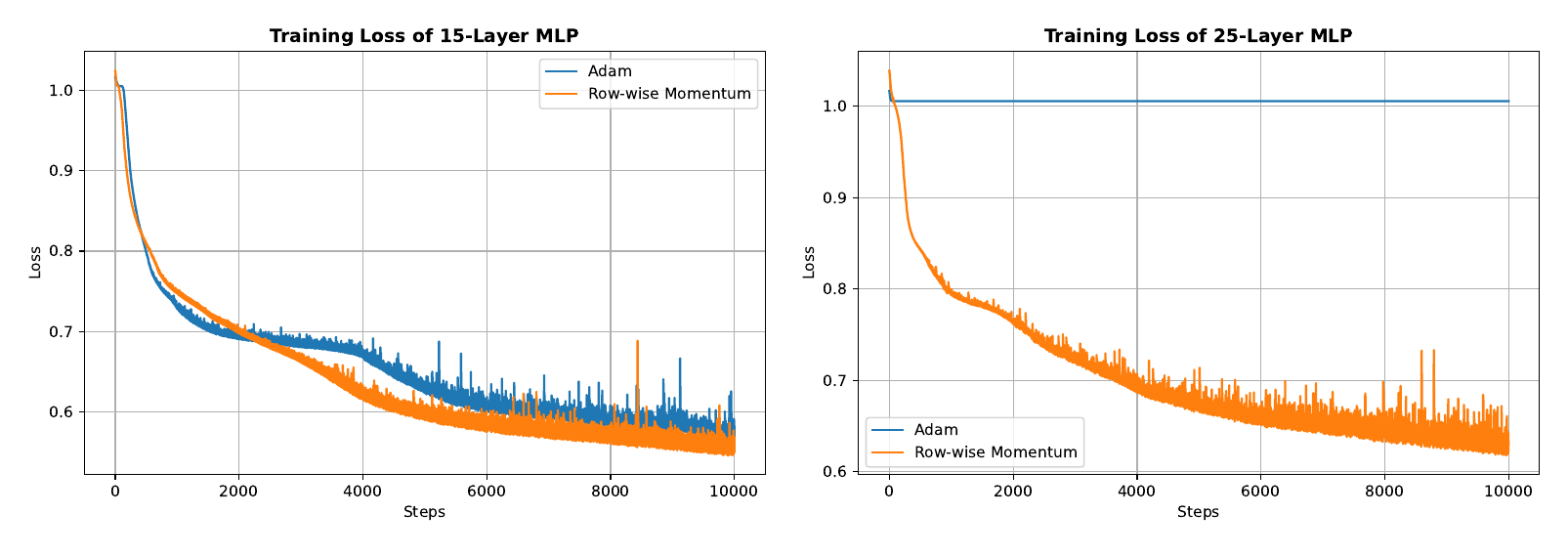}
    \caption{Training loss of Adam and Row-wise Momentum on stacked MLPs of different depths with \(lr=0.001\). Compared with \(lr=0.01\), both methods exhibit reduced loss fluctuations at $15$ layers, while entry-wise scaling still fails to train at $25$ layers.
    }
    \label{fig3}
\end{figure}

Under exactly the same experimental configuration, we additionally reduce the learning rate from \(0.01\) to \(0.001\) and evaluate both optimizers across different network depths. The experimental results are shown in Figure~\ref{fig3}. At $15$ layers, both methods exhibit substantially smaller loss fluctuations, while Row-wise Momentum still achieves a lower training loss than Adam. At $25$ layers, Row-wise Momentum remains able to train effectively, whereas Adam still completely fails to train. These results further indicate that row-wise scaling supports stable training at a learning rate $10\times$ larger than a rate at which entry-wise scaling already fails to train. This highlights a key advantage of row-wise adaptive scaling: its ability to maintain stable training at larger learning rates and greater network depths.

\end{document}

%% file: math_commands.tex
\usepackage{amsmath,amsfonts,bm}

\def\eqref#1{equation~\ref{#1}}

\def\1{\bm{1}}

\DeclareMathAlphabet{\mathsfit}{\encodingdefault}{\sfdefault}{m}{sl}
\SetMathAlphabet{\mathsfit}{bold}{\encodingdefault}{\sfdefault}{bx}{n}

\DeclareMathOperator*{\argmin}{arg\,min}